\documentclass[a4paper,fleqn]{cas-sc}
\usepackage[authoryear]{natbib}

\usepackage{amsmath,amssymb}
\usepackage{graphicx}
\usepackage{xcolor}
\usepackage{lineno}
\usepackage{booktabs}
\usepackage[LGR,T1]{fontenc}  
\usepackage[utf8]{inputenc}
\newcommand{\greek}[1]{{\fontencoding{LGR}\fontfamily{cmr}\selectfont #1}}
\usepackage{pifont}

\usepackage{wasysym}
\usepackage{tikz}
\usetikzlibrary{positioning,arrows.meta,fit}
\usepackage{algorithm}
\usepackage{algpseudocode}
\usepackage{longtable}
\usepackage{xurl}

\newcommand{\sys}{Kairos}
\newif\ifanonymous\anonymousfalse
\newcommand{\repo}[1]{\ifanonymous\emph{link withheld for anonymous review}\else\url{#1}\fi}

\newcommand{\yes}{\textcolor{green!55!black}{\ding{51}}}
\newcommand{\half}{\textcolor{orange}{$\LEFTcircle$}}
\newcommand{\no}{\textcolor{red!75!black}{\ding{55}}}

\newtheorem{definition}{Definition}
\newtheorem{proposition}{Proposition}
\newtheorem{example}{Example}

\begin{document}

\shorttitle{Kairos: Fast and Reliable Symbolic Task Planning for Autonomous Robots}
\ifanonymous\else\shortauthors{F. Mart\'in Rico et al.}\fi

\title[mode = title]{Kairos: Fast and Reliable Symbolic Task Planning for Autonomous Robots}

\ifanonymous\else
\author[1]{Francisco Mart\'in Rico}
\cormark[1]
\ead{francisco.rico@urjc.es}
\credit{Conceptualization, Methodology, Software, Formal analysis, Investigation, Data curation, Writing -- original draft, Funding acquisition}
\author[1]{Esther Aguado Gonz\'alez}
\ead{esther.aguado@urjc.es}
\credit{Validation, Investigation, Writing -- review and editing}
\author[1]{Francisco Miguel Moreno Olivo}
\ead{franciscom.moreno@urjc.es}
\credit{Validation, Software, Writing -- review and editing}
\author[1]{Rodrigo P\'erez Rodr\'iguez}
\ead{rodrigo.perez@urjc.es}
\credit{Validation, Resources, Writing -- review and editing}
\author[1]{Jos\'e Miguel Guerrero Hern\'andez}
\ead{josemiguel.guerrero@urjc.es}
\credit{Validation, Visualization, Writing -- review and editing}
\author[2]{Francisco Javier Rodr\'iguez Lera}
\ead{fjrodl@unileon.es}
\credit{Validation, Visualization, Writing -- review and editing}

\affiliation[1]{organization={Intelligent Robotics Lab, Universidad Rey Juan Carlos},
  addressline={Camino del Molino 5}, city={Fuenlabrada}, postcode={28942}, state={Madrid},
  country={Spain}}
\affiliation[2]{organization={Robotics Group, Universidad de Le\'on},
  addressline={Campus de Vegazana s/n}, city={Le\'on}, postcode={24071}, country={Spain}}

\cortext[1]{Corresponding author. Escuela de Ingenier\'ia de Fuenlabrada, Universidad Rey Juan Carlos, Camino del Molino 5, 28942 Fuenlabrada, Madrid, Spain.}
\fi

\begin{abstract}
Robots that plan while they act need a task planner that reads the models their
designers write, answers within milliseconds, and says what each answer means. Models are
written in the Planning Domain Definition Language (PDDL), which has grown from classical
planning to durative actions, numeric resources, time windows and continuous change, while
planners cover one fragment each and differ in what their failures mean. The contribution in
artificial intelligence is \sys, a single planner, not a portfolio of existing
planners, that covers the whole language and derives what each answer guarantees from the
search that produced it. We prove when the exhaustion of a search shows that no plan exists,
and give a relaxation of temporal tasks that proves unsolvability where temporal search
cannot. For the compression of durative actions, we give a criterion on the values that
actions read, prove that compression is sound, and show that it is not complete, so its
failure is never reported as a proof. An independent validator checks every plan. The
engineering application is on-line task planning for robots: \sys\ is open source, packaged
for the Robot Operating System, and the default planner of a widely used planning framework. On 1\,756 public tasks from nine suites, \sys\ solves more tasks than the
strongest evaluated planner of each suite within 100 milliseconds, 1 second and 30 seconds; it proves ten of eleven unsolvable tasks, and handles each
of the seventeen language features probed, where no other planner handles more than eight.
\end{abstract}

\begin{keywords}
Automated planning \sep Temporal planning \sep Plan validation \sep Unsolvability \sep
Robot task planning \sep Autonomous robots
\end{keywords}

\renewcommand{\printorcid}{}
\maketitle

\section{Introduction}
\label{sec:intro}

Automated planning gives an autonomous system the ability to decide what to do next from a model
of its actions and a goal. Its most common input language, PDDL \citep{pddl}, which
inherited its core from STRIPS \citep{strips}, has grown over almost three decades into a
family that describes durative actions, numeric resources, time windows, preferences and
continuous change \citep{pddl21,pddl22,pddl3,pddl31,pddlplus}, and the planners that read it have
become fast and reliable on each of these fragments. In robotics, planning is increasingly done
on line: frameworks such as PlanSys2 \citep{plansys2}, ROSPlan \citep{rosplan} or
Unified Planning \citep{up} write a PDDL problem from the robot's current knowledge, call a
planner, execute and monitor the plan, and call the planner again when something fails.

Planning on line puts three demands on the planner that offline evaluation tends to separate.
\begin{itemize}
\item \emph{Coverage of the language.} A robot domain mixes durative actions, batteries and other
numeric resources, time windows and sometimes continuous change. Rewriting it to fit a planner
is error-prone, and has to be redone each time the model changes.
\item \emph{Latency.} The planner runs inside the robot's control loop, on its own computer,
often to re-plan after a failure. An answer in tens of milliseconds is worth more than a better
plan in a minute.
\item \emph{Trust.} An invalid plan sends a robot to execute an impossible action, and a planner
that gives up without a reason leaves the designer guessing.
\end{itemize}
Today these demands are met one fragment at a time: each planner covers part of PDDL
(Section~\ref{sec:related}), so systems pick one planner and inherit its gaps.

Planning with large language models (LLMs) addresses a different question: LLMs
are not reliable plan generators on their own, and the approaches that produce correct plans
use them to write or refine a PDDL model for a symbolic planner, or pair them with a sound
verifier \citep{valmeekam2023planning,kambhampati2024llmmodulo}. This article studies the
planner and its guarantees on given models; how the models are obtained is out of its scope.
\sys\ (\greek{kair'os})\footnote{In Greek, \emph{kair\'os} is the opportune moment, as opposed to
\emph{chronos}, measured time: a planner should answer when the answer is still useful.} is the
system that implements and tests the approach.

This article asks whether a single planner can cover the PDDL family, from STRIPS to PDDL+,
stay as fast as the specialised planners on their own fragments, and state precisely what each
of its answers guarantees. The third demand is where the research questions lie. No single
search covers the family, and several of the searches that make a planner fast on a fragment
are sound but not complete: a temporally compressed task may have no plan when the task has
one, a search that memoises temporal states may discard the only partial plan that could be
completed, and a discretised hybrid task misses plans between its steps. A planner that
combines such searches must therefore answer two questions that a single-fragment planner
rarely states: which searches may be applied to a task, and what it has proved when they end
without a plan. The approach rests on three principles, which organise the article and its
discussion:
\begin{enumerate}
\item \emph{Soundness, completeness and performance are separate properties of each engine.}
Temporal compression can be sound without being complete; a search can end without proving
anything; a plan can be valid under a numerical model without being certified. The meaning of
each answer is derived from the properties of the engine that produced it, not assumed.
\item \emph{Semantic applicability precedes performance.} The semantics of the task (its
numeric fragment, its temporal structure, its size) first decides which transformations and
engines are applicable at all, and which of them are worth trying first even though they are
incomplete; only then are time and memory shared among them. An engine that does not apply is
never confused with a search that did not succeed.
\item \emph{Breadth need not cost responsiveness.} Covering many fragments could impose a
large overhead against specialised planners; we test, on the suites evaluated, whether it does.
\end{enumerate}

\sys\ is the planner that realises these principles. It is one planner, written from scratch,
and \emph{not a portfolio of existing planners}: it calls no external planner. Its five search
engines are built on one front end, one ground representation, one validator and one budget of
time and memory (Section~\ref{sec:architecture}), which is what lets a guarantee be stated for
the planner as a whole and not per component. The contributions of the article are:
\begin{enumerate}
\item \emph{A contract for the answers, with proofs} (Section~\ref{sec:guarantees}). We state
and prove when the exhaustion of a search shows that the task has no plan, in the presence of
dead-end pruning by a relaxation and of the merging of states that differ in unread numeric
variables (Proposition~\ref{prop:m1}), and when it shows that the plan in hand is optimal
without an admissible heuristic (Proposition~\ref{prop:optimal}). For temporal tasks, where
none of the fast engines is complete, we introduce the \emph{task without times}, a relaxation
that drops durations and keeps the order of happenings, and prove that its exhaustion is a
proof of unsolvability (Proposition~\ref{prop:untimed}). Every other failure is reported as
\emph{no plan found}, with its reason.
\item \emph{A value-based analysis of action compression} (Section~\ref{sec:profile}).
Compressing durative actions and searching without time is an old technique, and
\citet{coles2009inference} gave conditions under which an action can be compressed inside a
temporal search without losing completeness. We study the complementary case, in which a whole
task is handed to a non-temporal search: we give a criterion on the values that other actions
read (\emph{simple by value}), which accepts actions with end deletes and numeric end effects
that compression safety rejects, prove that every plan of the compression can be scheduled
into a valid temporal plan (Proposition~\ref{prop:m2sound}), and show with a counterexample
that compression under this criterion is \emph{not} complete (Example~\ref{ex:lift}). The
criterion is therefore used only to order the engines, and the meaning of a failure is recovered
by the task without times.
\item \emph{A validator independent of the planner's transformations}
(Section~\ref{sec:validator}), which checks every plan against the parsed PDDL, and a
comparison with VAL \citep{val} on 3\,587 plans in which every disagreement is traced to a
difference of convention or to a reproducible defect.
\item \emph{Evidence that breadth and responsiveness are compatible}
(Sections~\ref{sec:engines}, \ref{sec:orchestration} and~\ref{sec:experiments}): with the
engines selected from the semantics of the task, never from domain names, one planner solves
more tasks than the strongest evaluated planner of each suite at 100\,ms, 1\,s and 30\,s.
\item \emph{An open benchmark} (Section~\ref{sec:benchmark}) of public tasks pinned by commit
and checksum, new families for features that no public collection covers, each with unsolvable
instances so that wrong answers can be counted, and one-feature probes that measure what each
planner supports in practice, as opposed to what it declares.
\end{enumerate}
We are explicit about what is not claimed. The heuristics and search algorithms of the engines
build on FF, LAMA, Fast Downward, BFWS, POPF and Powerlifted, which we cite where they are
used (Appendix~\ref{app:heuristics} states what \sys\ computes and where it differs); action
compression is not new; and the shares of time among search configurations were set by hand,
not learned, so we claim nothing about portfolio scheduling. The contributions are the
contract of the answers and its proofs, the analysis of compression, the validator, and the
measurement of what this design costs and gives. The
experiments (Section~\ref{sec:experiments}) measure each of these: coverage and latency against
the best planner of each fragment, the support of each planner, an ablation of the engines and
components, the agreement of the validator with VAL, and the answers on tasks
without a plan.

The rest of the article is organised as follows. Section~\ref{sec:background} gives the preliminaries and
Section~\ref{sec:related} the related work. Sections~\ref{sec:architecture}
to~\ref{sec:validator} describe the planner: architecture, task profile, engines, orchestration,
answers and validation. Section~\ref{sec:implementation} describes the implementation and the
platform, Section~\ref{sec:benchmark} the benchmark, and Section~\ref{sec:experiments} the
experiments. Section~\ref{sec:cases} works through two examples, and
Section~\ref{sec:limits} discusses the limitations.

\section{Preliminaries}
\label{sec:background}

\paragraph{PDDL and its fragments.} We use the standard family: STRIPS with
types and the ADL features of the first PDDL, namely conditional effects and
quantified and disjunctive preconditions; numeric fluents and durative actions with duration
inequalities and continuous effects, from PDDL~2.1; derived predicates and
timed initial literals, from PDDL~2.2; preferences and trajectory
constraints, from PDDL~3; object fluents, from PDDL~3.1;
and processes and events, from PDDL+. Timed initial fluents are the numeric counterpart of timed initial
literals. We follow the semantics of PDDL~2.1 for temporal plans: happenings at the same
instant are evaluated on the state before them, interfering happenings must be separated by
$\epsilon$, and over-all conditions hold on the open interval of an action. For PDDL+ we follow
\citet{pddlplus}: processes change fluents continuously while their conditions hold, and
events fire as soon as theirs become true.

\paragraph{Compression and required concurrency.} \emph{Compressing} a durative action
replaces it by an instantaneous one with its weakest precondition and its strongest effects; a
plan of the compressed task is sequential, and is scheduled afterwards. This is the action
model of TGP \citep{tgp}, and many temporal planners search it. It is incomplete in
general: \citet{cushing2007} call a task \emph{temporally simple} when it can be solved
without concurrency, and show that tasks with \emph{required concurrency} (an action needs
another to be running) have no plan under compression. Such tasks need search over start and
end \emph{snap actions}, as in CRIKEY \citep{crikey} and POPF \citep{popf}, with a
simple temporal network (STN) \citep{dechter1991stn} to keep the partial plan schedulable.
Section~\ref{sec:related} reviews the conditions under which compression is known to be safe.

\paragraph{Notation.} A ground SAS$^+$ task has finite-domain variables $V$, numeric
variables $N$, actions $A$ and axioms. For an action $a$, $\mathit{pre}(a)$ is its
precondition and $\mathit{eff}(a)$ its effects; for a durative action, $\mathit{pre}_\vdash$,
$\mathit{pre}_\leftrightarrow$, $\mathit{pre}_\dashv$ are its start, over-all and end
conditions, and $\mathit{eff}_\vdash$, $\mathit{eff}_\dashv$ its start and end effects.

\section{Related Work}
\label{sec:related}

\paragraph{Planners and what they cover.} Classical planning is dominated by heuristic search
over a grounded task: FF \citep{ff}, Fast Downward \citep{fd} and LAMA \citep{lama},
with novelty-based search such as BFWS \citep{bfws,lipovetzky2012width}. Lifted planners
such as Powerlifted \citep{powerlifted} avoid grounding for tasks with millions of actions.
Numeric planners extend relaxation heuristics to numbers: Metric-FF \citep{metricff} and
ENHSP with numeric subgoaling \citep{enhsp,scala2020subgoaling}; ENHSP also plans with PDDL+
processes and events. Temporal planners follow two families: forward search over snap actions
with an STN or an LP, as POPF and, after it, OPTIC \citep{optic} and COLIN \citep{colin} do, and search over
compressed or temporal states, as TFD \citep{tfd} and LPG \citep{lpg} do; NextFLAP \citep{nextflap} adds
numeric constraints solved with SMT. Hybrid planners discretise time and validate, as UPMurphi,
DiNo \citep{upmurphi,dino} and ENHSP do, or compile to SMT, as SMTPlan+ does
\citep{smtplan}.

What each of these planners accepts is usually documented per requirement flag, and
documentation and behaviour do not always agree. Table~\ref{tab:support} therefore reports
\emph{operational} support, measured on one-feature tasks (Section~\ref{sec:support} gives the
method and what it does not show). No planner other than \sys\ handles more than eight of the
seventeen features probed: classical planners
stop at numbers and time, numeric ones at durative actions, temporal ones at derived
predicates, object fluents or processes, and hybrid ones at ADL and PDDL3. Several planners also
crash on, or return invalid plans for, features that they declare.

\begin{table*}[t]
\centering\footnotesize
\setlength{\tabcolsep}{2.5pt}
\resizebox{\textwidth}{!}{%
\begin{tabular}{@{}lccccccccccc@{}}
\toprule
Feature & LAMA & P.lifted & M-FF & ENHSP & POPF & OPTIC & TFD & LPG-td & NextFLAP & SMTPlan+ & \sys \\
\midrule
conditional effects & \yes & \no & \yes & \yes & \no & \no & \half & \no & \half & \no & \yes \\
quantified preconditions & \yes & \no & \yes & \yes & \yes & \yes & \yes & \yes & \half & \half & \yes \\
disjunctive preconditions & \yes & \yes & \yes & \yes & \half & \half & \yes & \yes & \half & \half & \yes \\
derived predicates & \yes & \no & \no & \no & \no & \no & \half & \yes & \no & \no & \yes \\
numeric fluents & \no & \no & \yes & \yes & \yes & \yes & \yes & \yes & \yes & \yes & \yes \\
non-linear numeric conditions & \no & \no & \no & \yes & \no & \no & \yes & \yes & \yes & \yes & \yes \\
object fluents & \no & \no & \no & \no & \no & \no & \no & \no & \yes & \no & \yes \\
durative actions & \no & \no & \no & \no & \yes & \yes & \yes & \yes & \half & \yes & \yes \\
required concurrency & \no & \no & \no & \no & \yes & \yes & \yes & \no & \yes & \yes & \yes \\
duration inequalities & \no & \no & \no & \no & \yes & \yes & \no & \no & \yes & \no & \yes \\
timed initial literals & \no & \no & \no & \no & \yes & \yes & \no & \yes & \half & \half & \yes \\
timed initial fluents & \no & \no & \no & \no & \no & \no & \no & \no & \no & \no & \yes \\
linear continuous effects & \no & \no & \no & \no & \yes & \half & \no & \no & \no & \yes & \yes \\
processes and events & \no & \no & \no & \yes & \no & \no & \no & \no & \no & \yes & \yes \\
non-linear continuous change & \no & \no & \no & \yes & \no & \no & \no & \no & \no & \no & \yes \\
preferences & \no & \no & \no & \yes & \yes & \yes & \no & \no & \no & \no & \yes \\
trajectory constraints & \no & \no & \no & \no & \half & \half & \no & \no & \half & \half & \yes \\
\midrule
supported & 4 & 1 & 4 & 8 & 8 & 7 & 6 & 7 & 5 & 6 & 17 \\
\bottomrule
\end{tabular}}
\caption{PDDL feature support, measured by execution rather than by declaration (Section~\ref{sec:support}). P.lifted: Powerlifted;
M-FF: Metric-FF. Each planner ran, for
10\,s, on 26 trivial tasks in the style of the IPC domains: one per feature, plus a variant
with durative actions of each non-temporal one. \yes: it solved the feature's task with a plan
that the validator accepts (for preferences, one that satisfies the preference). \half: partial
or doubtful: it crashed or returned an invalid plan on the feature's task, or solved only
another task that uses the feature. \no: it rejected every task that uses the feature, or gave
no answer. Symbols differ in shape as well as colour.}
\label{tab:support}
\end{table*}

\paragraph{Action compression and when it is safe.} Compression has been used in three ways.
\emph{As the action model}: TGP \citep{tgp} and the planners that adopt its semantics
plan with compressed actions only, and the temporal baseline of IPC~2008 solves the compressed
task with a non-temporal planner and adds the times afterwards \citep{colin}; these
planners cannot solve tasks with required concurrency \citep{cushing2007}. \emph{Inside a temporal search}:
\citet{coles2009inference} identify \emph{compression-safe} actions, for which compression
loses neither soundness nor completeness: its end conditions are among its over-all
conditions, its end deletes nothing and has no numeric effect, and its numeric start effects
do not depend on its duration. CRIKEY3, POPF, COLIN and OPTIC still split
such an action into its start and its end, and keep it in the temporal search with its STN,
but insert the end when its effects are needed or its invariants would be violated, so that
the action costs one search decision and not two \citep{colin}. The analysis is per
action, and the search remains temporal. \citet{itsat} define compression safety on sets of
actions and use it to prune a SAT encoding. \emph{As a compilation to classical planning}:
TPSHE \citep{tpshe} compiles a temporal task into a classical one, compressing every action
that is not a \emph{single hard envelope}, and characterises classes of tasks for which the
compilation is sound and complete.

\sys\ uses compression in the third way, and differs from this work in three respects
(Section~\ref{sec:profile} gives the definitions and the comparison). First, the question it
asks is about the \emph{task}: whether the whole task can be given to a non-temporal search
with the heuristics of the ground engine, not which ends a temporal search may skip. Second,
its criterion looks at the \emph{values} that other actions read, not at the syntactic form of
the action's end, so it accepts actions that compression safety rejects (an action that
deletes at end what it added at start, when no other action reads it; an action that takes a
numeric resource at start and returns it at end) and, conversely, it accepts tasks that
compression safety would correctly send to temporal search. Third, and as a consequence, we do
not claim completeness: we prove that the criterion does not give it
(Example~\ref{ex:lift}), never take the exhaustion of the compressed search as a proof, and
recover proofs of unsolvability by a separate relaxation (Proposition~\ref{prop:untimed}). The
earlier work asks when compression is safe; this article asks what a planner may answer when
it uses compression beyond the cases where it is.

\paragraph{Algorithm selection and portfolios.} \emph{Algorithm selection} chooses, per
instance, the component expected to perform best, from features of the instance
\citep{satzilla,ibacop,kerschke2019selection}. \emph{Portfolio scheduling} divides a budget
among several components, statically \citep{fdss,seipp2015portfolios} or per instance;
planning portfolios such as Fast Downward Stone Soup and IBaCoP schedule planners or
configurations for the same fragment, so that every component can, in principle, solve every
task, and the research question is how the schedule is optimised or learned for expected
coverage. \sys\ is not a portfolio in this sense, and this article makes no contribution to
that question. Its engines are not interchangeable solvers of the same tasks: they cover
different fragments, with different properties, on one shared representation. Which engines may
run on a task is therefore not a matter of expected performance to be learned, but a matter of
semantics to be decided: an engine cannot represent the task (ground classical search cannot
read durative actions), or it can but is incomplete for it (compressed search,
Section~\ref{sec:profile}), and its properties decide what its failure means
(Section~\ref{sec:guarantees}). Only the shares of time among the search configurations of one
engine are a scheduling question in the usual sense; they are few, set by hand
(Section~\ref{sec:orchestration}), and could be learned with the methods above.

\paragraph{Unified interfaces.} Unified Planning offers one API over many
planners, each with its own coverage; it is complementary to \sys, which could be one of its
engines.

\paragraph{Planning with large language models.} The LLM-based approaches that guarantee
correct plans translate the problem into PDDL for a symbolic planner
\citep{liu2023llmp,guan2023worldmodels}, synthesise generalised plans checked on PDDL tasks
\citep{silver2024generalized}, or pair the LLM with external verifiers
\citep{kambhampati2024llmmodulo}; they rely on a planner or a validator for PDDL, a role
that \sys\ can fill but that this article does not evaluate.

In summary, \sys\ differs from portfolios in what is selected and why: its
engines cover different fragments, so applicability is decided by the task's semantics (its
fragment and temporal profile) before performance is considered; and the engines share one
pipeline, one validator and one memory budget, so that a phase can hand over to another
(grounding to lifted search) when the measured memory requires it. It differs from the
single-fragment planners in handling every feature of Table~\ref{tab:support} with one
contract for its answers, and from unified interfaces, which call external planners with
their own guarantees, in that this contract is derived and enforced inside one planner. Table~\ref{tab:design} summarises the differences in design.

\begin{table}[t]
\centering\small
\resizebox{\textwidth}{!}{%
\begin{tabular}{@{}lllccc@{}}
\toprule
System & Fragments & Components & Budget & Validator & Unsolvable \\
\midrule
FD Stone Soup, IBaCoP & classical & planners, configurations & time & no & some \\
OPTIC, POPF & temporal, numeric & one planner & -- & no & no \\
ENHSP & numeric, PDDL+ & configurations & -- & built in & some \\
Unified Planning & per engine & external planners & -- & yes & per engine \\
\sys & all of Table~\ref{tab:support} & one planner, five engines & time, memory & yes & derived per engine \\
\bottomrule
\end{tabular}}
\caption{Design differences with the closest systems. Budget: what the components share;
Validator: whether the system validates plans itself; Unsolvable: which answers of
unsolvability are proofs (\sys: Table~\ref{tab:guarantees}).}
\label{tab:design}
\end{table}

\section{Architecture}
\label{sec:architecture}

\begin{figure}
\centering
\resizebox{\linewidth}{!}{%
\begin{tikzpicture}[
  box/.style={draw, rounded corners, align=center, font=\small, inner sep=3pt, minimum height=9mm},
  eng/.style={box, fill=gray!12, minimum width=24mm},
  >={Stealth[length=2mm]}]
\node[box] (parse) at (0,0) {parse, check,\\dialects};
\node[box] (norm) at (3.2,0) {normalise\\(ADL, PDDL3)};
\node[box] (ground) at (6.4,0) {ground (Datalog)\\200\,ms probe};
\node[box] (sas) at (9.6,0) {SAS$^+$,\\invariants};
\node[box] (profile) at (12.6,0) {task\\profile};
\node[eng] (lift) at (0,-2.2) {lifted\\search};
\node[eng] (m1) at (3.2,-2.2) {M1\\ground search};
\node[eng] (m2) at (6.4,-2.2) {M2\\compressed};
\node[eng] (m3) at (9.6,-2.2) {M3\\snap + STN};
\node[eng] (ts) at (12.6,-2.2) {time steps\\(hybrid)};
\node[draw, dashed, rounded corners, inner sep=4pt, fit=(lift)(ts),
      label={[font=\footnotesize, anchor=north east]south east:orchestrator: phases under one budget of time and memory}] (orch) {};
\node[box] (post) at (3.2,-4.4) {deorder,\\schedule};
\node[box] (val) at (7.4,-4.4) {validator\\(parsed PDDL)};
\node[box] (ans) at (11.6,-4.4) {plan $|$ unsolvable $|$\\no plan found + why};
\draw[->] (parse) -- (norm);
\draw[->] (norm) -- (ground);
\draw[->] (ground) -- (sas);
\draw[->] (sas) -- (profile);
\draw[->] (profile) -- (orch.north -| profile);
\draw[->, dashed] (ground.south west) -- node[above, sloped, font=\footnotesize, pos=0.5] {too large} (lift.north east);
\draw[->] (orch.south -| post) -- (post);
\draw[->] (post) -- (val);
\draw[->] (val) -- (ans);
\draw[->, dashed] (parse.west) -- ++(-0.5,0) coordinate (c1) -- (c1 |- 0,-5.6)
  -- node[below, font=\footnotesize, pos=0.6] {parsed PDDL} (val.south |- 0,-5.6) -- (val.south);
\end{tikzpicture}}
\caption{The pipeline of \sys. The orchestrator runs phases over the engines that the task
profile allows, all under one budget of time and memory; grounding hands over to the lifted
engine when it does not fit. The validator reads the parsed PDDL, not the planner's task.}
\label{fig:pipeline}
\end{figure}

Figure~\ref{fig:pipeline} shows the pipeline. Each stage receives an immutable structure and
produces another, so stages are tested in isolation, and there is no global state: two
planning requests in two threads share nothing.

\paragraph{Front end.} A hand-written parser reads PDDL~3.1 and PDDL+ and common dialects
(ENHSP's global constraints, SMTPlan's syntax), with diagnostics that point to file, line and
column and suggest close names. Normalisation compiles quantifiers, disjunctive preconditions
and trajectory constraints into derived predicates, conditional effects and monitoring
automata.

\paragraph{Grounding and the budget.} Grounding is a Datalog fixpoint; each relation reserves
memory in a shared \emph{budget} before it grows, so grounding stops cleanly at the limit
instead of being killed. The ground task becomes a SAS$^+$ task with numeric variables through
invariant synthesis, in the spirit of Fast Downward's translator. States are
bit-packed and stored once, in fixed blocks that are never reallocated. Search structures (state registry,
open lists, novelty tables, landmark bitsets) account for their memory in the same budget.

\section{The Task Profile}
\label{sec:profile}

The profile decides which engines apply. It records the numeric fragment (none, simple,
linear, non-linear), whether the task has derived predicates, timed initial literals or
fluents, duration inequalities, continuous effects, processes or events, the size of the
ground task, and whether the temporal part is simple. Temporal simplicity has the largest
effect on speed, because compressed search (M2) is much faster than snap-action search
(M3) when it applies.

A check \emph{by variable} marks a durative action as possibly requiring concurrency when it
changes a variable at start and again at end, and another action reads that variable. Under
this check, every move of a robot appears to require concurrency: \texttt{(at start (not (at ?r
?from)))} and \texttt{(at end (at ?r ?to))} leave the robot at no location while it moves, and
other actions read the robot's location. Yet no action needs the robot to be at no location.
Checking the values that are read, rather than the variables, removes these false alarms.
The definitions below make this precise; we then compare them with compression safety
\citep{coles2009inference}, which asks a related question about single actions.

\begin{definition}[Transient value]
Let $d$ be a durative action and $v \in V$ a variable. The value $x$ of $v$ is \emph{transient
for $d$} if $\mathit{eff}_\vdash(d)$ assigns $v := x$ and $\mathit{eff}_\dashv(d)$ assigns $v$
a different value; that is, $v = x$ holds only while $d$ runs. For a numeric variable $n$,
$d$ \emph{lowers $n$ transiently} if $\mathit{eff}_\vdash(d)$ decreases $n$ and
$\mathit{eff}_\dashv(d)$ increases it back (and symmetrically for \emph{raises}).
\end{definition}

\begin{definition}[Simple by value]
A temporal task is \emph{simple by value} if it has no timed initial literals or fluents and, for every durative action $d$:
\begin{enumerate}
\item for every value $x$ of $v$ transient for $d$, no condition of another action or axiom
reads $v = x$, $v \neq w$ with $w \neq x$, or compares $v$ with another variable;
\item for every numeric variable that $d$ lowers (raises) transiently, every condition of other
actions on it is a lower (upper) bound, and no effect copies its value.
\end{enumerate}
\end{definition}

\paragraph{Relation to compression safety.} A durative action $d$ is
compression-safe \citep[Definition~3.1]{coles2009inference} when (i)
$\mathit{pre}_\dashv(d) \subseteq \mathit{pre}_\leftrightarrow(d)$, (ii) its end has no delete
effects, (iii) its end has no numeric effects, and (iv) its numeric start effects do not depend
on its duration. Their example is \texttt{drive-truck} of Driverlog, which is the robot's move
above: deletes at start, adds at end, conditions over all. The authors note that the
definition is stronger than needed: it would suffice that the end effects cannot interfere
with the conditions of the actions that may run in the interval, but deciding that in general
is as hard as planning, and they leave numeric end effects, such as the return of a resource,
as an open extension. Simple by value is a linear-time test in that direction, on the values
that are read; it gives up the completeness that their definition is designed to keep. The two
criteria answer different questions and neither contains the other.
\begin{itemize}
\item \emph{Simple by value accepts actions that are not compression-safe.} An action that sets
\texttt{(busy ?r)} at start and deletes it at end violates (ii). It is simple by value as long
as no other action requires \texttt{(busy ?r)}: the value is transient and nobody reads it, so
no action needs to run inside the interval. The same holds for a variable of the SAS$^+$ task
that takes an intermediate value while the action runs. An action that takes a unit of a
capacity at start and returns it at end violates (iii); it is simple by value while the other
actions only ask for enough of the capacity. The robot's move is accepted by both criteria,
and it is only the check by variable that rejects it.
\item \emph{Compression safety rejects tasks that simple by value accepts, and is right to.}
In Example~\ref{ex:lift} below, each action has an end condition that is not an over-all
condition, which violates (i), so neither is compression-safe, while the task is simple by value and does require
concurrency. Condition 1 of our definition only detects the concurrency that a transient
value forces; an action may also need another one to have \emph{started}, without any value
being transient.
\item \emph{The use is different.} Compression safety is a property of an action inside a
temporal search, which keeps its STN and its start and end happenings and remains complete.
Simple by value is a property of a task and decides whether it is given to a non-temporal
search with the full heuristics of the ground engine (Definition~\ref{def:compression}), in
the manner of the compilations to classical planning \citep{tpshe}.
\end{itemize}
Simple by value is thus a weaker guarantee bought for a wider applicability: it preserves
soundness (Proposition~\ref{prop:m2sound}) and gives up completeness
(Example~\ref{ex:lift}). The rest of this section and Section~\ref{sec:guarantees} are about
making that trade safe: the compressed search is never the last engine to run, its exhaustion
is never an answer, and unsolvability is proved by other means
(Proposition~\ref{prop:untimed}). A planner could also use both analyses, taking the
exhaustion of the compressed search as a proof when every action is compression-safe; we have
not proved that this holds for our compression and do not rely on it.

\begin{definition}[Compression]
\label{def:compression}
Let $d$ be a durative action with a fixed duration $\delta$, and let $S$ be the values that its
unconditional start effects assign. Its \emph{compression} $\bar d$ is the instantaneous action
whose precondition is $\mathit{pre}_\vdash(d) \land \mathit{pre}_\leftrightarrow(d)|_S \land
\mathit{pre}_\dashv(d)|_S$, where $c|_S$ is $c$ with every literal that $S$ satisfies removed,
every numeric term rewritten through the numeric start effects (a battery drained at start must
still be non-negative), and the whole precondition made unsatisfiable if $S$ contradicts a
literal; and whose effect applies $\mathit{eff}_\vdash(d)$ and then $\mathit{eff}_\dashv(d)$,
the end winning on a variable that both assign, with $?\mathit{duration}$ replaced by $\delta$
everywhere. Its cost is $\delta$ under a time metric. The \emph{compression} of a task replaces
each durative action by its compression; a duration given by bounds is fixed first, to the
bound that the effects need, or one action per bound when they depend on it.
\end{definition}

\noindent A sequential plan of the compression is turned into a temporal one by an STN over
the happenings: each step's end follows its start by $\delta$, and each happening follows by
$\epsilon$ the last earlier happening that writes what it reads or writes, or reads what it
writes, and by $0$ those whose over-all conditions it must not break; the earliest times give
the schedule. Resources changed only by constant increases and decreases are first left out of
these orderings, and a step keeps that freedom only if the prefix of the plan validates.

\begin{proposition}
\label{prop:m2sound}
Let a temporal task have fixed durations, no timed initial literals or fluents, no continuous
effects, processes or events, and no conditional effects at start. Then every sequential plan
of its compression can be scheduled into a valid temporal plan.
\end{proposition}

\noindent\emph{Proof.} Execute the compressed actions in order, each durative action from its
start to its end before the next one starts, and each instantaneous action $\epsilon$ after the
previous happening. No two actions overlap, so every happening is evaluated in the state that
the compressed action saw: the start conditions in the state before the action, the over-all
and end conditions in the state after the start effects (which the compression requires
unless the start achieves them), and the end effects after the start effects, as in the
compression. The STN keeps every ordering between happenings that read and write the same
variable, so the earliest schedule evaluates each happening in a state that agrees with the
sequential one on what it reads; the freedom left to resources is kept only when the validator
accepts it, and every scheduled plan is validated before it is returned, with the sequential
schedule as the fallback. \hfill$\square$

\noindent The hypotheses are those under which the sequential execution reproduces the
compressed action: timed literals and continuous change would make the state depend on when
the action runs, and a conditional start effect would change, after the compression's check,
what the end conditions read. M2 is applied only to tasks without timed literals, and to tasks
whose continuous effects can be fixed in advance; in the other cases every plan is still
validated before it is returned, so a violation of the hypotheses can cost a plan but never
deliver an invalid one.

The converse does not hold: a task simple by value may need actions to overlap.

\begin{example}
\label{ex:lift}
Two robots prepare to lift a table. \texttt{prepare(r1, r2)} lasts 10 time units, sets
\texttt{(grip r1)} at start and needs \texttt{(grip r2)} at end; \texttt{prepare(r2, r1)} is
symmetric, and the goal is that both end. No value is transient (each grip stays set), so the
task is simple by value. Every plan in which each action starts before the other ends is
valid, but no sequential plan exists: whichever action runs first ends before the other grips.
The compression has no plan, and exhausting it says nothing about the
task. Appendix~\ref{app:lift} gives the PDDL.
\end{example}

\noindent Simple by value therefore only selects the faster engine; it is not a completeness
criterion, and this is where it parts from compression safety. When M2 runs out of states without a plan, \sys\ does not answer
\emph{unsolvable}: it tries the relaxation below and, if that does not prove that there is no
plan, M3 searches the time left (on Example~\ref{ex:lift}, M3 finds it in milliseconds).

\begin{proposition}
\label{prop:untimed}
Let the \emph{task without times} be the snap-action task of M3 with time reduced to the order
of the happenings: start and end happenings in any order, the over-all conditions of the
running actions and the global constraints checked whenever time advances, an action with an
over-all condition and a duration surely above zero not ending in the instant it starts, the
timed initial literals in any order that keeps their own, and the goal required after all of
them. If no durative action overlaps itself and the task has no continuous effects, processes,
events or duration-dependent effects with non-fixed durations, then a task whose task without
times has no path to the goal has no plan.
\end{proposition}

\noindent\emph{Proof sketch.} Every valid temporal plan, read as its sequence of happenings
with an advance between distinct instants, is a path to the goal of the task without times:
each happening is applicable in the state that the plan reaches there, the over-all
conditions hold at every instant inside their intervals, and the timed literals occur in their
order. Durations, the separation and interference are only dropped, so the relaxation keeps
every plan; a search that exhausts it, under the hypotheses of Proposition~\ref{prop:m1},
therefore shows that the task has no plan. \hfill$\square$

\noindent The proposition does not need the task without times to be finite, only that a
search proves it has no path to the goal. Its propositional part is finite (each durative
action adds two happenings and one running flag), but numeric fluents may make its state
space unbounded, so a search of it need not terminate and need not be cheap. \sys\ searches it
with a complete configuration of M1 (eager greedy search with $h^{\mathrm{FF}}$), within a
quarter of the remaining time; only the exhaustion of its open list is taken as a proof, under
the hypotheses of Proposition~\ref{prop:m1}, and otherwise the attempt proves nothing.

The check is linear in the size of the durative actions: readers are counted per value,
variable and fluent instead of compared per pair (on a task with 12\,480 durative actions,
from 7.7\,s to 118\,ms). Under it, match-cellar, turn-and-open and temporal-machine-shop of IPC
2011 require concurrency, while the other nine temporal domains of IPC 2011 and all temporal
domains shipped with PlanSys2 are simple; under the check by variable, all of PlanSys2's
domains would go to M3. Section~\ref{sec:q3} measures what M2 adds by running each engine
alone on the temporal suites.

\section{Engines}
\label{sec:engines}

\sys\ has five engines, all implemented in \sys\ itself on the same ground task (no
external planner is called), each for a different fragment or size of task: M1 for classical and
numeric tasks on the ground task, M2 for temporal tasks that are simple by value, M3 for tasks
that may require concurrency, a time-step engine for hybrid dynamics, and a lifted engine for
tasks too large to ground. All of them share the front end, the validator and the memory
budget.

\paragraph{M1: heuristic search on the ground task.} (Appendix~\ref{app:heuristics} gives the
definitions.) Lazy and eager greedy best-first search
and weighted A$^*$, with FD-style alternation of open lists, preferred operators boosted on
progress, $h^{\mathrm{FF}}$ with numeric subgoaling, LAMA landmarks with reasonable orderings
\citep{hoffmann2004orderings}, and a novelty queue partitioned by unachieved goals and
landmarks as in BFWS, with numeric features. The anytime driver follows LAMA:
restarts of weighted A$^*$ with decreasing weights, pruned by the incumbent's cost. Two
additions matter in the experiments. \emph{Projection of unread numeric variables}: a numeric
variable that no condition, duration, effect value, axiom or goal reads (an accumulated cost)
is reset to its initial value in every registered state, so that states that differ only in
accumulated cost are merged; when it was introduced in development, coverage of the classical
suite rose from 117 to 140 of 220 (Section~\ref{sec:q3}).
\emph{Incremental goals}: on tasks with 20 or more propositional goals, the closest pending goal
in the relaxation is reached by a short, bounded search that keeps the goals already achieved,
which turns plateaus that grow with the number of agents into a sequence of small searches.

\paragraph{M2: compression and rescheduling.} Searching the compressed task with a
non-temporal planner and scheduling afterwards is an established scheme
\citep{tgp,tpshe}; what is specific to M2 is the class of tasks it is applied to, the
scheduling of happenings, and the status of its failure (Section~\ref{sec:profile}). For tasks
simple by value, each durative action becomes an instantaneous one, M1 searches the
compressed task, and a post-processor orders happenings (not whole actions) by last writers and
readers and computes the earliest schedule. Resources that are only increased and decreased
by constants do not order happenings by themselves, so concurrent use of a shared capacity is
kept whenever the resulting plan validates. To spread work among agents, a preferred operator
is replaced by its applicable variant with another argument that finishes earliest in a partial
schedule.

\paragraph{M3: snap actions over an STN\@.} For tasks that may require concurrency, as in POPF: start and end snap actions, a partial order from last writers and readers,
and an incremental STN \citep{cesta1996} that rejects inconsistent extensions and undoes
without copying. A timed initial literal constrains only the happenings that interfere with it;
deadlines from literals that become false for good prune states from which a goal cannot be
reached in time in a temporal relaxation. The search's $g$ is the makespan of the partial plan,
so restarts reduce makespan. States are memoised by logical state with the first (or an
earlier-ending) partial plan that reached them \citep{coles2016memo}.

\paragraph{Time steps for hybrid dynamics.} Tasks with processes, events, non-linear continuous
effects, or durations that must lie strictly between their bounds become a discrete task with
wait actions, as in discretise-and-validate planners such as UPMurphi and DiNo \citep{upmurphi,dino}. A short wait
advances 0.1\,s, a step chosen to balance the precision of the plans against the depth of the search; a \emph{long wait} advances to just before some numeric condition (of a
precondition, event, goal, process, global constraint or over-all condition) changes truth
value, so that a long stretch costs one step. Events are forced: when one is applicable, only
events are. Process effects are integrated with RK4 at each step. Every plan that the search
reaches is validated in continuous time (Section~\ref{sec:validator}), and a goal state whose
plan does not validate is not a goal. Non-linear conditions enter the relaxed graph linearised
around the current state.

\paragraph{Lifted search.} For classical tasks whose grounding does not fit the budget, a
lifted engine in the line of Powerlifted computes successors by joins
over the state, with a lazy lifted $h^{\mathrm{FF}}$, a width-1 novelty queue and a
preferred-operator queue. Its relaxed exploration is semi-naive with a fixed join order and
projects partial joins onto the variables still read, which avoids the quadratic growth of
naive rule evaluation (\texttt{drive-truck} in logistics: from $1750^2$ to $1750$ combinations
per layer). Successor lists reserve budget before growing.

\section{Orchestration}
\label{sec:orchestration}

The orchestrator turns the profile into a list of \emph{phases}, each an engine with a task, a
search configuration and a share of the remaining time (Algorithm~\ref{alg:orch}). To build
it, the orchestrator makes three kinds of decision, with a different status. First, \emph{hard applicability constraints} exclude the
engines that cannot represent the task's fragment: M1 for durative actions, the snap-action
engines for processes and events, the lifted engine for anything but classical tasks. Second,
\emph{semantic search guidance} prefers strategies that may be faster but are possibly
incomplete: a task simple by value goes to compressed search first, which does not make
compressed search complete (Example~\ref{ex:lift}), so its failure is never a proof and M3
follows. Third, \emph{heuristic scheduling} shares time and memory among the selected phases,
set by hand and not optimised (below). The first two are decided by the semantics of the task
and determine what the planner may answer (Section~\ref{sec:guarantees}); only the third is a
question of performance. Keeping the three apart prevents an engine that does not
apply from being confused with a search that did not succeed.

\begin{algorithm}[t]
\caption{Phases chosen from the profile (one thread).}
\label{alg:orch}
\begin{algorithmic}[1]
\State ground with a 200\,ms probe; if it does not fit, give the lifted engine $\min(5\%, 1\,\mathrm{s})$, then ground fully
\If{grounding fails, or the ground task uses more than $1/4$ of the memory limit}
  \State \Return lifted engine with the remaining budget
\EndIf
\If{$\ge 20$ propositional goals} \State phase: incremental goals (half of the time) \EndIf
\If{no durative actions, processes or events}
  \If{numeric} \State lazy LAMA (25\%), then LAMA with novelty
  \Else \State lazy LAMA (85\%), then eager search with novelty \EndIf
\ElsIf{processes, events, non-linear change or duration inequalities}
  \State time steps: four configurations of short and long waits (2\%, 5\%, 10\%, rest)
\ElsIf{simple by value and fixed durations}
  \State M2 for the first plan, then the M3 configurations with $g$ = makespan (also if M2 runs out of states)
\Else
  \State M3 configurations: lazy LAMA (20\%), eager LAMA (37.5\% of the rest), LAMA with novelty
\EndIf
\State with several threads: the same phases concurrently, each with the global budget
\end{algorithmic}
\end{algorithm}

\paragraph{How the shares were set.} The shares and thresholds were set by hand during
development on the public suites, by inspecting where each phase found its first plans, and
then fixed; no per-domain tuning is possible, since the orchestrator never sees domain names.
For instance, LAMA gives no first plan after 25\,s on the classical suite, so the last 15\% for
novelty costs it nothing, while novelty solves within 2.5\,s tasks LAMA misses in 30\,s.
We did not optimise the shares, and the current implementation has no option to vary them
one by one; their sensitivity is left for future work (Section~\ref{sec:limits}).

\paragraph{Memory.} The budget measures the process's memory, so all phases share one limit.
A ground search typically needs about three times the memory of its task; hence the
quarter-of-the-limit rule above, which hands over to the lifted engine before the ground search
runs out of memory.

\section{Answers and Guarantees}
\label{sec:guarantees}

A run ends with one of: a validated plan (with its cost, and whether it is proven optimal);
\emph{unsolvable}; \emph{no plan found} with the reason; or an error in the input, with its
location. Table~\ref{tab:guarantees} states what each answer means for each fragment.

\begin{table}[t]
\centering\small
\begin{tabular}{@{}p{3.2cm}p{2.0cm}p{3.6cm}p{3.6cm}@{}}
\toprule
Fragment & Engines & \emph{Unsolvable} only when & Optimality \\
\midrule
classical, numeric (no time) & M1, lifted & M1 or the lifted engine exhausts the reachable space (Prop.~\ref{prop:m1}) & proven when a restart pruning by cost exhausts (Prop.~\ref{prop:optimal}) \\
temporal, simple by value, fixed durations & M2, M3 & the task without times has no path to the goal (Prop.~\ref{prop:untimed}); never from M2 (Example~\ref{ex:lift}) & not pursued: search optimises makespan \\
temporal with required concurrency or timed literals & M3 & the task without times has no path to the goal (Prop.~\ref{prop:untimed}); never from M3 (memoisation is incomplete) & not plan cost \\
duration inequalities, continuous effects, PDDL+ & time steps & never: \emph{no plan found} & not plan cost \\
\midrule
any & any & or grounding proves a goal unreachable even ignoring deletes & \\
\bottomrule
\end{tabular}
\caption{What \emph{unsolvable} and \emph{optimal} mean, per fragment. Every returned plan has
passed the validator (Section~\ref{sec:validator}), under the semantics and tolerances stated
there.}
\label{tab:guarantees}
\end{table}

\begin{proposition}
\label{prop:m1}
Consider a \emph{complete} search configuration of M1: one that inserts every generated
successor into an open list it eventually empties, except for states eliminated by the
mechanisms (b) and (c) below. These are the greedy best-first searches, lazy and eager, with or
without novelty and preferred-operator queues (every successor also enters the regular queue),
and the weighted A$^*$ restarts; the incremental-goals phase is not, and its failure is never
taken as exhaustion. Assume that (a) successor generation is complete: every action of the
ground task applicable in a state yields its successor; (b) a state is discarded as a dead end
only if the goal is unreachable in the relaxation of Appendix~\ref{app:heuristics}, which
over-approximates reachability; (c) duplicate detection merges only states that agree on every
variable that some condition, effect, duration, axiom or goal reads. If a complete
configuration of M1 exhausts its open list in a search started from the initial state, then the
task has no plan. For the lifted engine, the same holds up to a collision of its 128-bit state
keys.
\end{proposition}

\noindent\emph{Proof sketch} (full proof in Appendix~\ref{app:proofs})\emph{.} Successor generation is complete (every applicable action yields
its successor) and duplicate detection merges only equal states. Two reductions could hide a
plan. (i) Dead ends reported by the heuristics: a state is reported as a dead end only if the
goal is unreachable in its relaxation, and the relaxation over-approximates reachability
(Appendix~\ref{app:heuristics}): facts are never lost, a linear condition is reached by any
constant change that moves it the right way, and any change that depends on the state, or of a
variable of a non-linear condition, reaches its condition at least once; property-based tests check that no dead end
is reported on states with a plan in fully explored spaces. The lifted engine detects duplicates by a 128-bit key without
comparing states, so for it the statement holds up to a key collision (Appendix~\ref{app:proofs}). (ii) Projection of unread numeric
variables: two states that differ only in variables that nothing reads have the same applicable
actions, the same successors up to those variables, and the same goal status, so merging them
preserves reachability of the goal. A variable that starts undefined is never projected, since
whether it is defined decides applicability. Exhaustion then visits every reachable state up
to this equivalence, none of which is a goal. \hfill$\square$

\noindent Hypothesis (a) rests on grounding, which keeps every action that is reachable in the
delete relaxation; (b) on the relaxation over-approximating reachability
(Appendix~\ref{app:heuristics}); (c) on the projection above. The proposition is about
what exhaustion means, not about termination: with numeric variables the reachable space may be
infinite, and a search that ends by time or memory proves nothing and is reported as such.

\begin{proposition}
\label{prop:optimal}
Let the task have no durative actions, and let every action have a constant, non-negative
cost: its change of a linear metric that is minimised, made of constant increases and
decreases outside conditional effects, or 1 without a metric. If a restart of M1 with a plan
of cost $C$ in hand exhausts its open list, then $C$ is optimal.
\end{proposition}

\noindent\emph{Proof sketch.} The restarts are weighted A$^*$ searches that reopen a closed
state reached with a lower $g$ and discard a successor only when $g \ge C$ or when it is a dead
end in the sense of Proposition~\ref{prop:m1}. With non-negative costs, every prefix of a plan
cheaper than $C$ costs less than $C$, so none of its states is discarded by the bound; dead
ends have no plan; and merging states that differ only in unread variables keeps the cheapest
$g$ of each class through reopening. Exhaustion therefore visits every state of every plan
cheaper than $C$, none of which ends in a goal. The heuristic only orders the search, so its
admissibility is not needed. \hfill$\square$

\noindent The proof holds for any weight, though in practice the restart that exhausts is the
last one, with weight 1. It does not apply to temporal tasks, where costs are durations and
their sum is not the makespan, so \sys\ never reports a makespan as optimal; nor to metrics
outside the form above, for which \sys\ searches with unit costs, does not prune, and says so.

The incremental goals, the time-step engine and M3 are incomplete, and their failure is never
reported as \emph{unsolvable}. Neither is a time-out, nor exhausting M2
(Example~\ref{ex:lift}).

\paragraph{Explanations.} When there is no plan, \sys\ says why: which engine gave up and
whether its exhaustion proves anything; that grounding exceeded the memory limit and which
option raises it; or, when the goal is unreachable even without deletes, which goal facts, the
chain of actions that would achieve them, and the first precondition that no action achieves.
Figure~\ref{fig:why} shows two real outputs.

\begin{figure}
\small
{\footnotesize
\begin{verbatim}
$ kairos deliver.pddl no-corridor.pddl
note: the goal needs atoms that no action can reach
(in cup office) can never be reached:
  drop adds it, but needs (at ?r - robot office), which never holds
    move adds it, but needs (connected ?x - room office), which never holds
      no action adds it, and the initial state does not have it
kairos: unsolvable (exhausted)

$ kairos cellar.pddl long-fuse.pddl
note: engine M3 found no plan; as POPF, it keeps one partial plan per state,
  so this is strong evidence but not a proof
note: no plan: the engines that apply ran out of states (19 states) before
  the time limit, so more time does not help; none of them is complete for
  this task, so a plan may still exist (see the note before).
kairos: no plan found (exhausted)
\end{verbatim}
}
\caption{Two answers without a plan (stderr). Top: a delivery task whose goal room is not
connected; the chain of reasons ends at a static fact. Bottom: a fuse that takes 6\,s to mend
while each match burns for 5\,s; M3 exhausts its space, which proves nothing, so the answer is
\emph{no plan found}, with the reason.}
\label{fig:why}
\end{figure}

\section{Validation}
\label{sec:validator}

\paragraph{Independence.} The validator is independent of the search and of the planner's
transformations, but not of its reading of PDDL: it shares the parser and the semantic checks
with the planner, and nothing after them. It evaluates the parsed domain and problem, not the
normalised, ground or compressed task, so a bug in normalisation, grounding, compression or
scheduling cannot validate its own plans; a misreading of the input by the parser would affect
both. Comparing with VAL, which shares nothing with \sys, addresses that remaining dependence
(Section~\ref{sec:val-agreement}). Every plan is validated before it is returned; a plan that
fails is an internal error and is never delivered.

\begin{table}[t]
\centering\small
\begin{tabular}{@{}p{3.6cm}p{9.6cm}@{}}
\toprule
Fragment & A validated plan is \\
\midrule
classical, numeric (no time) & valid under the semantics of PDDL, as parsed, in floating point \\
temporal & valid under the semantics of PDDL~2.1, with $\epsilon = 0.001$ and a duration tolerance of 0.001 \\
PDDL+, linear change & valid under ENHSP's convention for steps at the same instant, with exact integration up to floating point \\
non-linear dynamics & accepted by a numerical integration (RK4, steps of 0.01\,s, bisection on watched conditions); not certified \\
\bottomrule
\end{tabular}
\caption{What validation guarantees for each fragment. Validity refers to the PDDL model; it
does not imply that the plan can be executed by a physical system.}
\label{tab:validity}
\end{table}

\noindent Table~\ref{tab:validity} summarises what a validated plan means; the paragraphs below
give the details.

\paragraph{Semantics.} Sequential and temporal plans follow VAL's semantics:
happenings at the same instant (within a tolerance) are evaluated on the state before them,
with deletes before adds; interfering happenings in the same instant are an error, and closer
than $\epsilon = 0.001$ a warning (VAL does not require it either); over-all conditions are
checked on the open interval; durations are checked with a tolerance of $0.001$, the precision
of plan files. PDDL3 trajectory constraints are checked on the sequence of states, and the
metric counts preference violations. For PDDL+ plans we follow the convention of ENHSP's plans,
where steps written at the same instant apply in the order written: a step that switches a pump
off and a step that needs it off can share an instant. This lets an action depend on the
immediate effect of another one at the same instant, which matters in PDDL+: with a separation
between them, the processes that are active would keep changing fluents in between, and the
second action would see values that the first one did not leave. VAL evaluates both steps on the state
before them and rejects such a plan, which under its semantics would need an $\epsilon$ between
them (Section~\ref{sec:val-agreement}).

\paragraph{Continuous change and its limits.} Linear continuous effects are integrated exactly
between happenings. Non-linear dynamics and processes are integrated numerically with
fourth-order Runge--Kutta in steps of 0.01\,s; when a watched condition (of a process, an event,
a global constraint or an over-all condition) changes truth value within a step, bisection
locates the instant, and events fire there, in cascades of at most 1000 rounds. Conditions of
processes and events are compared with a relative tolerance of $10^{-12}$ (rounding only), not
the tolerance of actions, so that two events on either side of a threshold do not both fire.
This is a \emph{numerical} validation, not a certified one: a condition that becomes true and
false again within one step, or a tangency, can be missed. Conditions that oscillate without
time passing (Zeno behaviour) are detected after eight consecutive bisections shorter than 1\%
of the step, and are then sampled at every step instead of located. For the fragments with
linear change the verdict is exact up to floating point; for non-linear ones it holds under the
step and tolerances above. The step is fixed in the current implementation, so the sensitivity
of the verdicts to it is not measured here (Section~\ref{sec:limits}).

\paragraph{Why not VAL\@.} VAL is the reference, and we compare against it
(Section~\ref{sec:val-agreement}). We did not reuse it, for three reasons. \emph{Coverage}: its
latest release (2021) rejects object fluents as a syntax error and stops on non-linear
continuous change (``differential equations of this nature are not handled'') while exiting
with status 0, so a caller that trusts the exit status accepts an unchecked plan; it also
rejects some valid PDDL, such as a single numeric effect not wrapped in \texttt{and}.
\emph{Integration}: \sys\ validates every plan before delivering it, so validation must be a
library call that shares nothing mutable, needs no files and does not throw; VAL is a command.
\emph{Diagnostics}: the validator reports the failing step, its line in the plan file, the
condition and the values that falsify it, e.g.\ \texttt{(>= (battery r1) 10): 5 >= 10 is
false}.

\section{Implementation and Platform}
\label{sec:implementation}

\paragraph{Software.} \sys\ is about 36\,000 lines of C++20, plus 19\,000 of tests, under the
Apache~2.0 license. It has no dependencies beyond the C++ standard library and threads, and in
particular it does not depend on ROS: it builds with standard CMake (\texttt{cmake -S . -B build}
and \texttt{cmake --build build}; CMake 3.22 or later, and a C++20 compiler such as GCC~11 or
later) and installs a CMake package, so that another project links it with
\texttt{find\_package(kairos)}. Each stage of Figure~\ref{fig:pipeline} is an internal library;
the public library exports only an API that is stable within a major version.

\paragraph{A library and a program.} \sys\ is meant to be used in two ways. As a \emph{library},
a program plans in its own thread (\texttt{Planner::solve}) or starts an anytime search in a
thread of the library (\texttt{Planner::start}) and receives each better plan through a
callback, waits for the first one, or stops the search; no exception crosses the API, and every
answer carries its status and diagnostics. This is how a robot's executive embeds it, without
the cost of starting a process or writing files. As a \emph{program}, the command
\texttt{kairos DOMAIN PROBLEM} writes plans in IPC, POPF or JSON format, with an exit code for
each kind of answer, and \texttt{kairos validate} and \texttt{kairos check} validate plans and
PDDL on their own; this is how it runs in planning pipelines and in the experiments of this
article.

\paragraph{Integration.} \sys\ is also distributed as ROS~2, conda, Python and planutils
packages, which contain the same library and program, and it is the default planner of PlanSys2.
The source code, the benchmark, the scripts that produce every table and figure of this
article, and the raw results are public (Appendix~\ref{app:resources}).

\paragraph{Testing.} The test suite has 979 tests: unit tests of every stage, property-based
tests of the heuristics on fully explored state spaces (admissibility where claimed, no false
dead ends), equivalence of the normalised and the original tasks on random walks, and
end-to-end planning with validation. Continuous integration runs them under address,
undefined-behaviour and thread sanitizers, with fuzzing of the parser, inside a 512\,MB memory
cgroup, with a line coverage of at least 90\%, and builds the ROS~2 packages for every target.

\paragraph{Experimental platform.} The experiments ran on an Intel Core Ultra 9 285H (6
performance cores up to 5.4\,GHz, 8 efficient cores up to 4.5\,GHz, 2 low-power cores;
24\,MiB L3) with 60\,GB of RAM, under Ubuntu 26.04 with Linux 7.1. \sys\ was built with GCC 15.2
(\texttt{-O3}, release). The other planners run in containers: planutils 0.10.25 (OPTIC, TFD,
POPF, LPG-td, Metric-FF, ENHSP-2020, LAMA, SMTPlan+, VAL) and builds from source at pinned
commits (the IPC~2023 classical winners Scorpion Maidu, Levitron, Fast Downward Stone Soup 2023
and DecStar with their agile configurations, lama-first, Dual-BFWS, LNM-Plan, the current ENHSP,
NextFLAP, Powerlifted), and TAMER and Aries through Unified Planning. Commits and container
recipes are in the online appendix.

\section{Benchmark}
\label{sec:benchmark}

The benchmark is a repository of its own. Each suite is a list of files of public collections
pinned by commit and SHA-256, with a script that downloads exactly them (Table~\ref{tab:suites}).

\begin{table}[t]
\centering\small
\begin{tabular}{@{}lrl@{}}
\toprule
Suite & Tasks & Content \\
\midrule
Classical & 220 & IPC 1998--2023, 10 per domain \\
Numeric & 420 & IPC 2002, 2023, 2026, 10 per domain \\
Temporal & 225 & IPC 2002--2014, 5 per domain \\
Concurrency & 55 & required concurrency; temporal numeric \\
PDDL3 & 125 & IPC 2006 preferences and constraints \\
PDDL+ & 255 & hybrid domains; SMTPlan benchmarks \\
Durations & 25 & durations strictly between bounds (new) \\
Hard to ground & 287 & lifted planning benchmarks \\
PlanSys2 at scale & 144 & 8 robot domains, 10--20\,000 objects (new) \\
Per feature & 450 & 28 features, with unsolvable tasks (new) \\
Probes & 26 & one feature each, for Table~\ref{tab:support} (new) \\
\bottomrule
\end{tabular}
\caption{Benchmark suites: 2\,232 tasks in three groups. The nine comparative suites (1\,756
tasks, from Classical to PlanSys2 at scale) are run by every planner that reads them
(Tables~\ref{tab:coverage} and~\ref{tab:stats}); the feature bank (450 tasks) checks \sys's
answers per feature, including the tasks without a plan of Section~\ref{sec:q5}, and is not
compared; the 26 probes serve only Table~\ref{tab:support}.}
\label{tab:suites}
\end{table}

\paragraph{New suites.} \emph{PlanSys2 at scale} takes the eight domains shipped with PlanSys2,
unchanged, and generates problems at six sizes from about 10 to 20\,000 objects, three seeds
each. The generator builds each problem together with a reference plan, by a scripted walk of
each agent through the problem (no planner is involved), and the plan is validated, so every
problem is solvable by construction. The reference plans only certify solvability: they are
neither used by any planner nor used to select problems, and the sizes and seeds are fixed in
advance, so no problem was kept or dropped according to how a planner did on it. The \emph{feature bank} has, for each of the 28 features the planner lists, up to six
families of public problems that use it, five tasks each from easy to hard, and new families
where none exists (object fluents, timed initial fluents, a robotic warehouse with up to 400
robots). The feature bank is not one of the comparative suites: lama-first, OPTIC, POPF and ENHSP also
ran it, but they reject or fail on between 56\% and 87\% of its tasks, since most of its features are
outside what they read, so a count of solved tasks would compare coverage of the language
rather than performance. It serves to check \sys's answers per feature, in particular on the
tasks without a plan (Section~\ref{sec:q5}). \emph{Unsolvable instances} are built so that their unsolvability has a short argument
(a resource that cannot reach a threshold, a deadline before the earliest arrival, a goal
behind a fact no action adds), which serves as their oracle; Section~\ref{sec:q5} reports the
answers on them.

\paragraph{Probes.} The 26 probes cover the 17 features listed in Table~\ref{tab:support}: one task per feature (17), plus a durative-action variant for each of the nine non-temporal features (9), to accommodate planners that only
support durative actions. The probes are trivial tasks following the style of IPC domains: they are typed, use parameterized actions, and contain a fixed fact in the initial state and a single fact in the goal.

\section{Experimental Evaluation}
\label{sec:experiments}

We answer five questions. \textbf{Q1}: Does \sys\ match or beat the best planner of each
fragment in coverage, at short and long limits? \textbf{Q2}: Which features does each planner
support in practice? \textbf{Q3}: What does each component and the orchestration contribute?
\textbf{Q4}: Does the validator agree with VAL, and how fast is it? \textbf{Q5}: Are the answers
on unsolvable tasks correct?

\subsection{Method}
\label{sec:method}

The platform and the versions of all planners are those of Section~\ref{sec:implementation}.

\paragraph{Limits and protocol.} 30\,s of wall time and 3\,GB of memory per task, enforced by a
cgroup (PlanSys2 at scale: 60\,s and 2\,GB; results with different limits are never mixed).
Each run is limited to one CPU; eight runs proceed at a time, and runs are not pinned to a
core, so on this hybrid processor a run may be scheduled on a slower core. This introduces
timing variability that may affect planners differently, depending on how the operating system
schedules their runs; it is the main reason why we compare at three limits and over domains
rather than on single times. Each configuration ran once per task. \sys\ runs with one
thread (the command uses up to four by default).

\paragraph{Clock.} Times are wall times from the start of the planner's process to its first
plan (for the 100\,ms and 1\,s columns) or to the end of the run, measured inside the container,
so that the container's start is not counted, and \emph{not corrected}: the cold start of the
Apptainer images in which planutils packages POPF, SMTPlan+ and ENHSP-2020 (0.08--0.19\,s on a
one-action task) is counted, and the counts without it are given in parentheses in
Table~\ref{tab:fullcoverage}. Process start, parsing and grounding count for every planner.
For planners that stop at their first plan (the agile configurations, lama-first, POPF, ENHSP,
Metric-FF, LNM-Plan, DecStar, Dual-BFWS, SMTPlan+, Powerlifted, TAMER) this is their first-plan
time; anytime planners that do not report it get ``--'' in the short columns. \sys\ reports
the time of its first validated plan, leaving out its process start (about 5\,ms), since this overhead is not incurred when \sys\ is used as a library by another program.

\paragraph{Validation and quality.} Every plan of every planner is checked by \sys's validator,
and an invalid plan counts as unsolved (Section~\ref{sec:val-agreement} checks this validator
against VAL). Quality is the IPC score at the limit: for each task, the best cost among the
planners that solved it divided by the planner's own cost.

\paragraph{Statistics.} Tasks of one domain are correlated, so we resample domains, not tasks:
for the difference in solved tasks between \sys\ and the best other planner of each suite, a
95\% bootstrap interval (10\,000 resamples) and a two-sided sign test over the domains where
they differ. These account for the variation across domains, not for run-to-run timing
variability, since each configuration ran once per task; we report them as descriptive
comparisons per suite, not as a global test of superiority.

\paragraph{Versions.} \sys\ is release 1.0.6. The other planners ran on the same machine, with
the same limits and the same scripts.

\subsection{Q1: Coverage and latency}

\begin{table}[t]
\centering\small
\setlength{\tabcolsep}{4pt}
\begin{tabular}{@{}llrrr@{}}
\toprule
Suite (tasks) & Planner & 100\,ms & 1\,s & limit \\
\midrule
Classical (220) & \sys & \textbf{76} & \textbf{111} & \textbf{166} \\
 & LAMA (anytime) & -- & -- & 160 \\
 & lama-first & 0 & 75 & 159 \\
 & Dual-BFWS & 57 & 88 & 133 \\
\midrule
Numeric (420) & \sys & \textbf{168} & \textbf{216} & \textbf{299} \\
 & ENHSP & 0 & 128 & 208 \\
 & LNM-Plan & 22 & 69 & 117 \\
\midrule
Temporal (225) & \sys & \textbf{102} & \textbf{147} & \textbf{180} \\
 & LPG-td & -- & -- & 125 \\
 & POPF & 0 & 82 & 104 \\
\midrule
Concurrency (55) & \sys & \textbf{25} & \textbf{32} & \textbf{33} \\
 & POPF & 0 & 18 & 24 \\
 & TFD & -- & -- & 24 \\
\midrule
PDDL3 (125) & \sys & \textbf{59} & \textbf{71} & \textbf{86} \\
 & OPTIC & -- & -- & 45 \\
\midrule
PDDL+ (255) & \sys & \textbf{55} & \textbf{94} & \textbf{148} \\
 & SMTPlan+ & 0 & 31 & 35 \\
 & ENHSP & 0 & 18 & 23 \\
\midrule
Durations (25) & \sys & \textbf{24} & \textbf{24} & \textbf{25} \\
 & POPF & 0 & 17 & 17 \\
\midrule
Hard to ground (287) & \sys & \textbf{48} & \textbf{156} & \textbf{215} \\
 & Powerlifted & 22 & 106 & 190 \\
\midrule
PlanSys2 at scale (144) & \sys & \textbf{60} & \textbf{77} & \textbf{92} \\
 & TFD & -- & -- & 58 \\
 & POPF & 0 & 45 & 57 \\
\bottomrule
\end{tabular}
\caption{Tasks solved with a valid plan within 100\,ms, 1\,s and the limit, with uncorrected
times; per suite, \sys\ and the best other planners. Every planner, with the counts without the
start-up of Apptainer, is in Table~\ref{tab:fullcoverage}.}
\label{tab:coverage}
\end{table}

\begin{figure}
\centering
\includegraphics[width=\linewidth]{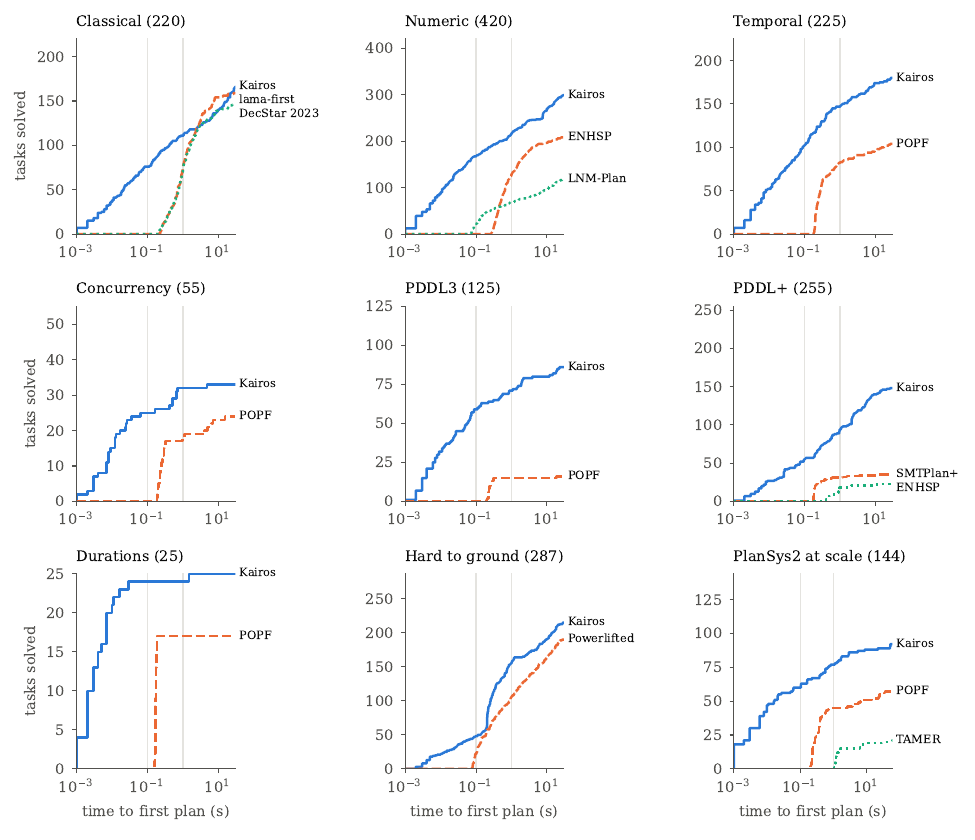}
\caption{Tasks solved against time to the first plan (log scale), per suite: \sys\ and the two
best other planners that report their first plan. Grey lines at 100\,ms and 1\,s.}
\label{fig:cactus}
\end{figure}

Table~\ref{tab:coverage} and Figure~\ref{fig:cactus} summarise the results. At the limit,
\sys\ solves the most tasks on every suite; at 100\,ms and 1\,s it solves the most among the
planners that report a first-plan time (anytime planners that do not report it, such as LAMA,
LPG-td, TFD and OPTIC, have no entry there). At the limit, the differences are smallest on
classical tasks (166 against LAMA's 160) and largest on the numeric (299 against ENHSP's 208),
PDDL3, PDDL+ and PlanSys2 suites. At short limits, within 100\,ms no planner built on Fast
Downward's translator solves any classical task, while \sys\ solves 76 of 220, and 168 of the
420 numeric ones. The curves show why: the other planners start to solve tasks after
0.1--0.3\,s, the time of their front ends, while \sys\ solves its first tasks in about a
millisecond.

The short limits must be read with care. The 100\,ms results measure end-to-end performance
under the protocol of Section~\ref{sec:method}, not isolated algorithmic speed: they include
the front end of each planner, the cold start of the Apptainer images of three planners
(0.08--0.19\,s, which the parenthesised counts of Table~\ref{tab:fullcoverage} remove), and they
leave out the 5\,ms of \sys's own process start. On the 1\,s and full-limit columns these differences
are small relative to the limit, and the conclusions of this section rest mainly on them; the
100\,ms column shows what an application that embeds the library obtains, which is the use
\sys\ is designed for.

\begin{table}[t]
\centering\small
\begin{tabular}{@{}llrrrr@{}}
\toprule
Suite & Best other & $\Delta$ solved & 95\% interval & domains $+$ / $-$ & sign test $p$ \\
\midrule
Classical & LAMA & +6 & [-7, +17] & 10 / 3 & 0.092 \\
Numeric & ENHSP & +91 & [+40, +143] & 24 / 6 & 0.001 \\
Temporal & LPG-td & +55 & [+31, +80] & 24 / 6 & 0.001 \\
Concurrency & POPF & +9 & [-1, +22] & 3 / 1 & 0.625 \\
PDDL3 & OPTIC & +41 & [+15, +67] & 15 / 6 & 0.078 \\
PDDL+ & SMTPlan+ & +113 & [+73, +155] & 22 / 0 & $<$0.001 \\
Durations & POPF & +8 & [+0, +16] & 2 / 0 & 0.500 \\
Hard to ground & Powerlifted & +25 & [+11, +39] & 20 / 2 & $<$0.001 \\
PlanSys2 at scale & TFD & +34 & [+18, +51] & 13 / 0 & $<$0.001 \\
\bottomrule
\end{tabular}
\caption{Difference in tasks solved at the limit between \sys\ and the best other planner,
with a bootstrap interval over domains, the domains where each solves more, and a two-sided
sign test over those domains.}
\label{tab:stats}
\end{table}

Table~\ref{tab:stats} accounts for the correlation within domains.
The gains are significant ($p < 0.01$, interval above zero) on the numeric, temporal, PDDL+,
hard-to-ground and PlanSys2 suites, and the interval is above zero on PDDL3 ($p = 0.08$). Against LAMA on
classical tasks, the gain at 30\,s ($+6$, $[-7, +17]$) is within the noise; there, the advantage
lies in latency: at 1\,s, 111 tasks against lama-first's 75. On the concurrency and durations
suites, with eleven and five domains, the sign test cannot reach significance.

\begin{table}[t]
\centering\small
\begin{tabular}{@{}lrrrr@{}}
\toprule
Suite & solved & p50 (s) & p90 (s) & p99 (s) \\
\midrule
Classical & 166 & 0.145 & 17.09 & 28.46 \\
Numeric & 299 & 0.053 & 9.94 & 25.52 \\
Temporal & 180 & 0.061 & 3.76 & 26.08 \\
Concurrency & 33 & 0.012 & 0.65 & 4.77 \\
PDDL3 & 86 & 0.026 & 2.04 & 22.40 \\
PDDL+ & 148 & 0.307 & 6.58 & 20.65 \\
Durations & 25 & 0.003 & 0.02 & 1.49 \\
Hard to ground & 215 & 0.282 & 11.40 & 23.04 \\
PlanSys2 at scale & 92 & 0.010 & 1.77 & 51.65 \\
\bottomrule
\end{tabular}
\caption{Time to \sys's first plan over the tasks it solves: median and 90th and 99th
percentiles.}
\label{tab:latency}
\end{table}

As for \sys's own latency, Table~\ref{tab:latency} gives the distribution of its first-plan times.
The median is below 0.15\,s on every suite but PDDL+ (0.31\,s) and hard to ground (0.28\,s),
and below 15\,ms on the concurrency, durations and PlanSys2 suites.

Plan quality can be compared across planners only where all of them report the same
metric, the classical and numeric suites (Table~\ref{tab:quality}). \sys\ has the highest IPC
score on both, marginally on classical tasks (154.5 against LAMA's 152.5) and widely on numeric
ones (292.7 against ENHSP's 166.5). Per solved task, LAMA's classical plans are slightly better
(0.95 against 0.93 on average), and \sys's numeric plans better than ENHSP's (0.98 against
0.80): the gain in score on classical tasks comes from solving more, not from better plans.

\begin{table}[t]
\centering\small
\begin{tabular}{@{}llrr@{}}
\toprule
Suite & Planner & IPC score & tasks with a cost \\
\midrule
Classical & \sys & 154.5 & 166 \\
Classical & LAMA & 152.5 & 160 \\
Classical & lama-first & 127.4 & 159 \\
Classical & Scorpion Maidu & 148.1 & 154 \\
Classical & Scorpion Maidu agile & 113.4 & 142 \\
Classical & Levitron & 149.5 & 156 \\
Classical & Levitron agile & 112.7 & 141 \\
Classical & FDSS 2023 & 149.7 & 157 \\
Classical & FDSS 2023 agile & 110.6 & 142 \\
Classical & DecStar 2023 & 108.8 & 146 \\
Classical & Dual-BFWS & 108.1 & 133 \\
\midrule
Numeric & \sys & 292.7 & 299 \\
Numeric & ENHSP & 166.5 & 208 \\
Numeric & ENHSP-2020 & 83.4 & 103 \\
Numeric & LNM-Plan & 93.1 & 117 \\
Numeric & Metric-FF & 45.5 & 57 \\
\midrule
\bottomrule
\end{tabular}
\caption{Quality at the limit (IPC score) on the suites where all planners report the problem's
metric.}
\label{tab:quality}
\end{table}

\paragraph{Where \sys\ loses.} Another planner solves more tasks of a domain in 34 of the 291
domains (Table~\ref{tab:losses} in the appendix). The largest gaps are in rubiks-cube and parking
(LAMA), counters, ext-plant-watering, sailing-wind and ztalloc (ENHSP), floor-tile and storage
(LPG-td), turn-and-open (TFD), storage with preferences (OPTIC) and chasing cars (ENHSP).

\subsection{Q2: Support}
\label{sec:support}

Table~\ref{tab:support} (Section~\ref{sec:related}) answers Q2. Under the operational criterion
of the probes, \sys\ handles all 17 features, and no other planner handles more than eight. The doubtful cells are of three kinds, each checked by
hand on the planner's output: crashes on a feature the planner declares (OPTIC and POPF on
disjunctive preconditions; OPTIC on continuous effects), invalid plans (OPTIC and POPF start
both actions of a \texttt{sometime-before} constraint at the same instant; NextFLAP starts an
action at the instant a timed literal makes its condition true; TFD ignores the start condition
of a conditional effect), and features solved only inside another task (NextFLAP's durative
actions). Uncaught exceptions with which NextFLAP rejects an input count as rejections.

A probe shows that a planner reads a feature and handles it correctly on one small task; it
does not show that every use of the feature, or every combination with others, is supported.
Support for features one at a time would be of little use if each feature belonged to a
different planner, since a task that mixes two of them would then belong to none. This is not
how \sys\ is built. Conditional effects, quantified and disjunctive conditions, derived
predicates, object fluents, preferences and trajectory constraints are handled by the front
end, before an engine is chosen (Section~\ref{sec:architecture}), so every engine receives
them already normalised and they combine with any fragment by construction. The remaining
features decide the engine, and a task that mixes them goes to the engine that covers the most
general one (Algorithm~\ref{alg:orch}): the time-step engine plans processes, events and
non-linear change together with durative actions and timed initial literals, and M3 plans
required concurrency together with numeric resources, timed literals and linear continuous
effects. What a combination costs is a different matter, which the probes do not measure.
Combinations are exercised by the other suites (the feature bank mixes, for instance, object
fluents with durative actions and numeric resources with timed initial fluents, and the
PlanSys2 suite mixes durative actions, numeric resources and negative preconditions), where
every plan is validated; Table~\ref{tab:support} should be read as a lower bound on what each
planner cannot do, not as a certificate of full support.

\subsection{Q3: Ablation}
\label{sec:q3}

\begin{table}[t]
\centering\small
\begin{tabular}{@{}llrrr@{}}
\toprule
Suite & Variant & 100\,ms & 1\,s & limit \\
\midrule
Temporal & full & 102 & 147 & 180 \\
 & M2 only & 101 & 146 & 179 \\
 & M3 only & 49 & 86 & 124 \\
\midrule
Concurrency & full & 25 & 32 & 33 \\
 & M2 only & 5 & 5 & 5 \\
 & M3 only & 25 & 32 & 33 \\
\midrule
Classical & full & 76 & 111 & 166 \\
 & LAMA, no novelty fallback & 69 & 105 & 157 \\
 & lazy FF & 74 & 106 & 139 \\
\midrule
Numeric & full & 168 & 216 & 299 \\
 & LAMA, no novelty fallback & 165 & 212 & 272 \\
 & lazy FF & 172 & 220 & 265 \\
\midrule
PlanSys2 at scale & full & 60 & 77 & 92 \\
 & no incremental goals & 55 & 66 & 94 \\
\midrule
\bottomrule
\end{tabular}
\caption{Ablation: \sys\ against variants that run one engine or one search configuration
alone (\texttt{--engine}, \texttt{--search}), same limits and protocol.}
\label{tab:ablation}
\end{table}

Table~\ref{tab:ablation} runs \sys\ against variants with one engine or one configuration
alone. \emph{Compressed search is what makes the temporal suites fast}: M2 alone solves almost
what the full planner solves on the temporal suite (179 against 180), while M3 alone, the
POPF-like engine, solves 124, and 49 instead of 102 within 100\,ms. On the concurrency suite
the profile sends the tasks to M3, and M3 alone matches the full planner, while M2 alone solves
5. These rows also isolate the profile: M3 alone is what a classification that sends every
temporal task to M3 gives (a check by variable does so on all of PlanSys2's domains,
Section~\ref{sec:profile}), and M2 alone what one that sends every task to M2 gives; the
profile keeps the speed of the first on the temporal suite and the coverage of the second on
the concurrency suite. \emph{The novelty fallback} adds 9 classical tasks (157 to 166) and 27 numeric ones (272 to
299) at the limit; plain lazy FF solves 139 and 265. On the numeric suite, lazy FF alone is a
little faster at short limits (172 and 220 against 168 and 216): there, the breadth of the
schedule costs some latency. \emph{The incremental goals} matter for latency on PlanSys2 at
scale (77 against 66 tasks within 1\,s, 60 against 55 within 100\,ms), not for coverage at the
limit (92 against 94).

One component cannot be switched off from the command, and we report it separately because
its evidence is of a different kind: when projecting unread numeric variables was added during
development, the classical suite went from 117 to 140 tasks and the numeric suite from 253 to
289. These are measurements of two development versions that differ in that change, not a
controlled ablation of the released planner, and other components of those versions differ
from it.

\subsection{Q4: Agreement with VAL}
\label{sec:val-agreement}

We ran \sys's validator and VAL on the last plan of every run of \sys\ above
(1\,567 plans, all of them valid for \sys's validator) and on its broken variants: a step
removed, when the plan has more than one step (1\,514), and a timed step moved one unit earlier,
when some step is timed at 1 or later (506); 3\,587 plans in all. VAL decides 3\,045 of them. It
cannot decide the other 542: on 288 it stops with a parse or type-checking error (object
fluents, PDDL+ dialects, some IPC~2023 numeric domains), on 42 with ``validator failed'' or
``not handled'' on dynamics it does not support, and 212 are PDDL+ plans whose processes run on
after the last action, which VAL's plan format, where a plan ends at its last happening, cannot
express. On the 3\,045 the two validators agree on 3\,024 verdicts. Each of the 21
disagreements was checked by hand; the instances and plans are in the online appendix:
\begin{itemize}
\item 13 are the PDDL+ convention of Section~\ref{sec:validator}: plans of SMTPlan's car domain,
in the PDDL+ suite and in the feature bank, in which two \texttt{decelerate} steps share an
instant. Each reads and decreases the acceleration, so VAL rejects them as mutex; applied in
order, as ENHSP writes plans, they are valid. These are differences of convention, not errors
of either validator.
\item 5 are broken variants that VAL accepts although a step's strict numeric precondition is
false: in four variants of petri-net (IPC~2026) and one of sugar (IPC~2023) with a step removed,
VAL applies a later step whose precondition \texttt{(> (value ?p1) 0)} (petri-net) or
\texttt{(> (in-storage ?l ?b) 0)} (sugar) is $0 > 0$.
\item 3 are plans that VAL rejects with trajectory constraints of tpp (IPC~2006, instances 13
and 15, the latter also in a variant with a step moved): VAL reports a
\texttt{sometime-before} constraint whose antecedent \texttt{(> (stored ?g) 0)} it takes as
satisfied at time 0, where the problem sets \texttt{(stored ?g)} to 0 and the plans first store
those goods at times 2\,699 (instance 13) and 8\,053 (instance 15).
\end{itemize}
Run as commands, \sys's validator takes 1.4\,ms per plan in the median (99th percentile
71\,ms) and VAL 3.0\,ms (146\,ms).

On the 51 plans of the probes (the 26 plans, 20 variants with a step removed and 5 with a step
moved), VAL cannot decide five: it rejects the four with object fluents as a syntax error and
stops on non-linear change. The two validators agree on 45 of the other 46. The exception is the
PDDL+ convention again: \sys's plan switches a pump off and checks the tank at the same instant,
valid when steps apply in order (ENHSP's reading) and invalid when both are evaluated on the
state before them (VAL's). In development, \sys's validator also agreed with VAL on all 515
plans of POPF on PlanSys2 and IPC temporal problems and their broken variants.

\subsection{Q5: Tasks without a plan}
\label{sec:q5}

\begin{table}[t]
\centering\small
\begin{tabular}{@{}lccccc@{}}
\toprule
Task (family) & \sys & lama-first & OPTIC & POPF & ENHSP \\
\midrule
\texttt{object-fluents/delivery} & unsolvable & error & error & error & error \\
\texttt{object-fluents/delivery-time} & unsolvable & error & error & error & error \\
\texttt{object-fluents/elevator} & unsolvable & error & error & error & error \\
\texttt{object-fluents/keys} & unsolvable & error & error & error & error \\
\texttt{timed-initial-fluents/charger} & unsolvable & error & error & error & error \\
\texttt{timed-initial-fluents/market} & unsolvable & error & error & error & error \\
\texttt{timed-initial-fluents/traffic} & no plan found & error & error & error & error \\
\texttt{warehouse/capacity} & unsolvable & error & error & error & error \\
\texttt{warehouse/cells} & unsolvable & error & error & error & error \\
\texttt{warehouse/free} & unsolvable & error & error & error & error \\
\texttt{warehouse/time} & unsolvable & error & error & error & error \\
\bottomrule
\end{tabular}
\caption{Answers on the eleven generated tasks of the feature bank that have no plan. ``error'':
the planner rejected the task (none of them supports object fluents, timed initial fluents or
the warehouse's features).}
\label{tab:unsolvable}
\end{table}

Table~\ref{tab:unsolvable} shows the answers on the eleven tasks of the feature bank built
without a plan. \sys\ answers \emph{unsolvable} on ten of them: four by grounding, three by
the exhaustion of M1 (Proposition~\ref{prop:m1}) and three by the task without times
(Proposition~\ref{prop:untimed}); on the eleventh, a timed-initial-fluent task, M3 and the task
without times prove nothing and the answer is \emph{no plan found}. The other planners reject
these tasks. Outside the feature bank, \sys\ answered \emph{unsolvable} once more, by grounding,
on a settlers task of IPC~2023 whose goal has a fact that no action adds: eleven
\emph{unsolvable} answers in all, each with a short argument for the absence of a plan.
Over the whole benchmark, no planner gave
a wrong answer in the feature bank (a plan for a task without one, or \emph{unsolvable} for one
with a plan), and \sys\ never answered \emph{unsolvable} on a task that another planner
solved.\footnote{One development version did, on a PDDL3 task with \texttt{within} constraints
that OPTIC solves; the released planner answers \emph{no plan found} there, since M3's
exhaustion is not a proof.}

\section{Case Studies}
\label{sec:cases}

\paragraph{Required concurrency.} A fuse must be mended in the light of a match: the match
gives light only while it burns (5\,s), and mending (2\,s) needs light over all its duration.
The profile finds the transient value and sends the task to M3:
{\footnotesize
\begin{verbatim}
[orchestrator] engine M3: (light-match m1) sets (light) only while it runs,
               and other actions can use that
[orchestrator] engine M3: 3 durative actions as snap actions, 0 timed initial literals
0.000: (light-match m1) [5.000]
0.000: (mend-fuse f1) [2.000]
kairos: parse 0.1 ms, normalize 0.1 ms, ground 0.1 ms, sas 0.1 ms, search 1.0 ms,
        total 1.5 ms, first plan 0.9 ms
\end{verbatim}
}
Mending starts with the match: its over-all condition holds on the open interval
$(0, 2)$, after the match's start effect. With a 6\,s fuse the task has no plan, and the answer
is the one at the bottom of Figure~\ref{fig:why}.

\paragraph{Processes and events.} A pump fills a tank at 1 unit/s while it runs; an event
spills it at level 10; the goal is a full tank (level $\ge 5$) without spilling. The profile
sends the task to the time-step engine:
{\footnotesize
\begin{verbatim}
[orchestrator] PDDL+: 1 processes and 1 events, steps of 0.1 s
[orchestrator] search: lama-novelty, lama, then both with long waits (auto)
0.000: (pump-on r1) [0.000]
5.000: (pump-off r1) [0.000]
5.000: (check r1) [0.000]
5.000: @PlanEND
kairos: ... total 2.2 ms, first plan 1.1 ms
\end{verbatim}
}
A long wait crosses the five seconds in one step, stopping just before the goal's condition
becomes true. The validator integrates the process in continuous time and checks that the
event never fires.

\section{Limitations and Scope of the Conclusions}
\label{sec:limits}

\paragraph{What the guarantees cover.} On temporal tasks, \sys\ proves unsolvability only by grounding or by the task
without times (Proposition~\ref{prop:untimed}); M2, M3 and the time-step engine are incomplete,
and their failure is reported as \emph{no plan found}. M3's memoisation by logical state can
discard a partial plan with a better temporal profile, which also limits makespan improvement.
The proofs of Propositions~\ref{prop:m1} and~\ref{prop:optimal} rest on properties of the
implementation (complete successor generation, an over-approximating relaxation, exact
duplicate detection) that are tested, by property-based tests on fully explored spaces, but not
verified formally; the error found in the non-linear relaxation shows that such tests can miss
a case.

\paragraph{What validity means.} A validated plan is valid with respect to the PDDL model under
the semantics of Table~\ref{tab:validity}. For non-linear dynamics the validation is numerical,
not certified, and the sensitivity of its verdicts to the integration step has not been
measured. For PDDL+, \sys\ adopts ENHSP's convention for steps at the same instant, which VAL
does not share (Section~\ref{sec:val-agreement}); the plans are valid under that convention.
None of this says that a plan can be executed by a physical system.

\paragraph{What the experiments support.} The comparison covers nine suites, the planners and
versions of Section~\ref{sec:implementation}, one run per task, and a hybrid processor without
pinning to cores. Comparisons at three limits and over domains reduce, but do not remove, the
noise this adds. The 100\,ms results include asymmetries in start-up accounting
(Section~\ref{sec:experiments}) and should be read as end-to-end performance under this
protocol. The orchestration shares were selected by hand during development, using public
suites that also appear in the evaluation. Although the policy uses few coarse parameters and
no domain-specific tuning, this introduces a potential development-set bias, and the robustness
of these settings on an independent collection of domains has not been established; neither
has their sensitivity been measured. Feature support is operational: a
probe shows that one small task with the feature is handled, not that the feature is supported
in every combination. Finally, the relaxations ignore resource consumption in some cases (a
vehicle without fuel far from a station), and grounding memory limits tasks such as assembly
with many interchangeable parts.

\section{Conclusion}

We set out to build one planner for the PDDL family that states what each of its answers
guarantees, and to test three principles on it.

\emph{Soundness, completeness and performance are separate properties.} A planner that covers
the family must use searches that are sound but not complete, so the meaning of an answer
cannot be assumed; it has to be derived from the engine that produced it. Doing so gives three
kinds of answer: a validated plan, a proof of unsolvability and an inconclusive search. We
proved when the exhaustion of a search is a proof (Proposition~\ref{prop:m1}), when it
certifies optimality without an admissible heuristic (Proposition~\ref{prop:optimal}), and
that the task without times proves the unsolvability of temporal tasks for which no fast
engine is complete (Proposition~\ref{prop:untimed}); on the tasks built without a plan, ten of
the eleven answers are proofs of unsolvability and none is wrong. Every returned plan is validated against
the parsed PDDL, independently of the planner's transformations, although for non-linear
dynamics the validation remains numerical rather than certified.

\emph{Semantic applicability precedes performance.} Action compression is the clearest case.
It is an old technique, and compression safety \citep{coles2009inference} tells a temporal
search when it loses nothing. We asked what a planner may do beyond that: the criterion by
value sends to a non-temporal search tasks whose actions are not compression-safe, compression
remains sound (Proposition~\ref{prop:m2sound}), and Example~\ref{ex:lift} shows that it is not
complete, so the criterion orders the engines and never decides an answer. The ablation shows
what this buys: the speed of compressed search on the tasks without required concurrency and
the coverage of snap-action search on the others.

\emph{Breadth need not cost responsiveness.} On the nine comparative suites,
under our protocol, \sys\ solves more tasks than the strongest evaluated planner of each
suite at the full limits (30\,s, and 60\,s for PlanSys2 at scale), and the most among the
planners that report a first plan at 100\,ms and 1\,s;
on classical planning the gain over LAMA at 30\,s is not significant. These results are limited
to the suites, planners, versions and protocol evaluated (Section~\ref{sec:limits}).

Future work includes a proof of when the exhaustion of the compressed search is itself a proof
(for instance, when every action is compression-safe), a temporal state identity that keeps
temporal profiles, numeric and disjunctive landmarks, certified validation of non-linear
dynamics, learned shares for the search configurations with a study of their sensitivity, an
engine for Unified Planning, and
the use of \sys\ as both planner and critic in loops where a language model writes or repairs
PDDL models.

\section*{Data availability}
The benchmark, with the scripts, the raw results of every run and the plans, is deposited in
a public data repository (\repo{https://doi.org/10.5281/zenodo.24362419}), and the planner is
public (\repo{https://github.com/PlanSys2/kairos}).

\section*{Declaration of generative AI and AI-assisted technologies in the manuscript preparation process}
During the preparation of this work the authors used Claude (Anthropic) to assist in revising
and editing the text of the manuscript, to check its grammar and overall consistency, and to
help design and run the benchmark. After using this
tool, the authors reviewed and edited the content as needed and take full responsibility for
the content of the published article.

\ifanonymous\else
\section*{Funding}
This work was supported by the Spanish Ministry of Science, Innovation and Universities and
the State Research Agency (MICIU/AEI/10.13039/501100011033) and by ERDF/EU [grant numbers
PID2025-171682OB-C32 (CONECTALO), PID2024-161761OB-C22 (AURORAS)]; and by the Community of
Madrid [grant number TEC-2024/TEC-62 (iRoboCity2030-CM)].
The funding sources had no involvement in the study design, in the collection, analysis and
interpretation of data, in the writing of the article, or in the decision to submit it for
publication.

\printcredits

\section*{Declaration of competing interest}
The authors declare that they have no known competing financial interests or personal
relationships that could have appeared to influence the work reported in this article.
\fi

\bibliographystyle{cas-model2-names}
\bibliography{kairos}

\appendix

\section{Proofs}
\label{app:proofs}
The proofs concern the complete search configurations of Proposition~\ref{prop:m1}; the
incomplete phases (incremental goals, M2, M3, time steps) never report their failure as a
proof, so no proof is needed for them.

\paragraph{Proposition~\ref{prop:m1}.} Let $S_0$ be the initial state, $\mathit{succ}(s)$ the
states reached from $s$ by one applicable action, and $\sim$ the equivalence that ignores the
values of the projected numeric variables $P$: those that no precondition, effect value,
duration, axiom or goal reads, and that are defined in $S_0$.

\emph{(a) $\sim$ is a bisimulation for applicability and goals.} If $s \sim s'$, the two states
agree on every variable that a precondition, an axiom or the goal reads, so the same actions
are applicable, the same axioms fire and both or neither are goals. For an applicable $a$, its
effects on variables outside $P$ compute their values from variables outside $P$ (effects on
$P$ are never read), so $\mathit{succ}_a(s) \sim \mathit{succ}_a(s')$. A variable undefined in
$S_0$ is excluded from $P$ because an effect that changes an undefined fluent makes an action
inapplicable, so definedness is read.

\emph{(b) The search visits every reachable class.} The registry stores one representative per
class (it resets $P$ to its initial values). An expanded state generates all its applicable
actions (the successor generator is a decision tree over all preconditions, followed by the
checks of negations, numeric conditions and object comparisons). A generated state is
discarded only if its class was seen before, or if a heuristic reports it as a dead end. The
heuristics report a dead end only when the goal is unreachable in their relaxation (delete
relaxation with numeric subgoaling, Appendix~\ref{app:heuristics}), which over-approximates
reachability: a fact once reached stays reached; a linear condition is reached through any
constant change that moves it the right way, and through any state-dependent change or any
change of a variable of a non-linear condition at least once; the caps on repetitions bound the
cost of reaching a condition, never whether it is reached. So if the relaxation cannot reach
the goal from $s$, every condition left unreached is a fact that no action adds or a linear
condition that every change moves the wrong way by a constant, and no plan from $s$ exists. Hence, when the open list is empty,
every class reachable from $S_0$ through states that are not dead ends has been expanded.

\emph{(c) Conclusion.} If a plan existed, its states would form a path of classes from $S_0$ to
a goal class, none of which is a dead end; by (b) all of them would have been generated and the
goal class tested, and the search would have returned the plan. So exhaustion proves that no
plan exists. The lifted engine satisfies (b) in the same way, with two
differences. Its successors are all the groundings that satisfy the preconditions, computed by
joins. Its dead ends come from a relaxed exploration that stops when a layer adds no fact; when
that exploration exceeds its caps it falls back to the goal count, which reports no dead ends,
and when the budget runs out the search stops and its end is not taken as exhaustion. And
duplicates are detected by a 128-bit key of the state (two independent 64-bit hashes of its
atoms) without comparing the states, so the proof holds up to a collision of the key, whose
probability for $n$ states is below $n^2 / 2^{129}$ (below $10^{-22}$ for $10^8$ states).
\hfill$\square$

\paragraph{Proposition~\ref{prop:m2sound}.} Given in the main text. The scheduler is
not needed for soundness: the sequential schedule is valid on its own, and the earliest
schedule is kept only when the validator accepts it.

\paragraph{Proposition~\ref{prop:untimed}.} Given in the main text. The relaxation needs that
no durative action overlaps itself, as in POPF, since the snap task runs one instance of each at
a time; with continuous effects, processes, events or effects that depend on a non-fixed
duration it is not a relaxation, and \sys\ does not use it.

\section{Heuristics and Search of M1}
\label{app:heuristics}

The components of M1 come from FF, LAMA, Fast Downward and BFWS; this appendix states what \sys\
computes, so that its behaviour can be reproduced, and marks what differs from those planners.
A ground task is $\Pi = \langle V, N, A, s_0, G, c\rangle$: finite-domain variables $V$,
numeric variables $N$, actions $A$ with conditions on $V \cup N$ and effects that assign values
of $V$ and increase, decrease or assign numeric expressions, the initial state $s_0$, the goal
$G$ and non-negative costs $c$.

\paragraph{Relaxed exploration and $h^{\mathrm{FF}}$.} Facts are relaxed as in FF (an achieved
value is never lost). A linear numeric condition $e \ge 0$ (or $e > 0$) is a node of the
exploration, reached through any action that changes $e$ by $\Delta > 0$ in the current state,
counted $k = \lceil \mathit{gap}/\Delta \rceil$ times (or $\lfloor \mathit{gap}/\Delta \rfloor
+ 1$ for a strict condition), with additive cost $c(a) \cdot k$: this is numeric subgoaling in
the style of Metric-FF and ENHSP, with $\Delta$ evaluated in the state being evaluated. When
$\Delta$ depends on the state and does not help there, the action counts once. The relaxed plan
is extracted backwards from the goal, and $h^{\mathrm{FF}}(s) = \sum_a c(a) \cdot t(a)$, where
$t(a)$ is the largest number of repetitions that any condition of the relaxed plan needs from
$a$.

\paragraph{When the relaxation reports a dead end.} A state is a dead end when the goal node is
not reached. The exploration is an over-approximation of reachability: a fact, once reached,
stays reached; a linear condition is reached by any contributor whose constant change moves it
the right way; a contributor whose change depends on the state (a non-constant expression, a
scaling, an undefined value), and any contributor that changes a variable of a non-linear
condition, reaches that condition at least once, whatever the sign of its change in the current
state; a clock action that moves nothing now counts $10^6$ times, which is far but finite;
inequalities $\ne$, object comparisons and derived atoms read through numbers count as
satisfied. The caps bound costs, never reachability: when a condition would need more than
$10^9$ repetitions of a contributor (an increment of 1 towards $2 \cdot 10^9$), the condition
is reached with the cost of $10^9$ repetitions, which underestimates the distance but keeps the
state alive; the cap can make the heuristic less informed, never report a dead end. The only conditions that can
stay unreached are therefore linear ones that every contributor moves the wrong way by a
constant, or facts that no action adds, and both are unreachable in the task. Non-linear
conditions are linearised at each state as their tangent (gradient by finite differences) to
count repetitions. A tangent can have zero gradient (\texttt{x * y >= 4} at the origin) and
show no contributor helping; it is the rule above, that any contributor that changes a variable
of a non-linear condition reaches it at least once, that keeps such a condition reachable and
prevents a false dead end.

\paragraph{Landmarks.} Fact landmarks are found backwards from the goal as in LAMA's RHW
generation, without disjunctive landmarks: the first achievers of a landmark are those
reachable in the relaxation without it, and the facts shared by the conditions of all of them
are landmarks before it. Orderings are greedy-necessary, plus reasonable orderings detected by
the interference test of Fast Downward; reasonable orderings that form a cycle are dropped.
Numeric conditions are crossed but never become landmarks. The landmark heuristic is
path-dependent: a landmark is reached once it holds in some state of the path (the intersection
over the paths that reach a state), and is \emph{future} if it was not reached, or was reached
but is false now and is a goal or greedy-necessarily before an unreached landmark, or was
reached before one of its reasonable parents. $h^{\mathrm{LM}}(s)$ is the sum, over future
landmarks, of the cost of their cheapest achiever (the count with unit costs).

\paragraph{Novelty.} The features of a state are its facts, one per finite-domain variable, and
one per numeric variable holding its exact value (undefined is one more value), as in Panino.
States are partitioned by $\langle \#g(s), h^{\mathrm{LM}}(s) \rangle$, where $\#g$ is the number
of unachieved goals (by $\#g$ alone when there are no landmarks), as in BFWS. The width $w(s)$ is
1 if some feature of $s$ was never seen in its partition, 2 if some pair was never seen, and 3
otherwise. With more than 128 features only pairs that contain a feature changed from the
parent are checked, as in BFWS; seen tuples are kept as 64-bit hashes, and above 64\,MB of
pairs the width falls back to $\{1, 3\}$. The tables are never reset within a search.
Differences from BFWS(f5): novelty is one more queue of the alternation below, ordered by
$\langle w, h^{\mathrm{FF}}\rangle$, not a lexicographic key over all queues, and numeric
values are features without discretisation.

\paragraph{Search and open lists.} The first plan is searched by greedy best-first search, lazy
(deferred evaluation, one entry per operator with the parent's value) or eager. Each heuristic
has one queue with every entry and one with the entries reached by preferred operators (the
union of the preferred operators of all heuristics: the actions of the relaxed plan applicable
in the state, and those that achieve a future landmark). The queue with the lowest priority
counter is popped and its counter increased by 1; when some heuristic improves on its best
value, the preferred queues and the novelty queue get a boost of 1000, as in Fast Downward. The
configurations are: \texttt{ff}, lazy with $h^{\mathrm{FF}}$; \texttt{lama}, lazy with
$h^{\mathrm{FF}}$ and $h^{\mathrm{LM}}$; \texttt{eager-lama}, the same, eager;
\texttt{novelty}, eager with $h^{\mathrm{FF}}$ and the novelty queue; \texttt{lama-novelty},
eager with both heuristics and the novelty queue. As in LAMA, the first search uses unit costs
when the costs differ (cost plus one when some cost is zero). After the first plan, restarts of
weighted A$^*$ with weights 5, 3, 2 and 1, with reopening and without novelty, each prune by the
incumbent's cost; with pruning, the last weight repeats until a search exhausts its open list
(Proposition~\ref{prop:optimal}).

\section{The Lifted-Table Task}
\label{app:lift}

The task of Example~\ref{ex:lift}, as given to \sys\ and to the validator. The static
\texttt{partner} rules out \texttt{prepare(r1, r1)}.

{\footnotesize
\begin{verbatim}
(define (domain table-preparation)
  (:requirements :strips :typing :durative-actions :negative-preconditions)
  (:types robot)
  (:predicates (partner ?r ?other - robot) (grip-active ?r - robot)
               (ready ?r - robot))
  (:durative-action prepare-lift
    :parameters (?r ?other - robot)
    :duration (= ?duration 10)
    :condition (and (at start (partner ?r ?other))
                    (at start (not (grip-active ?r)))
                    (at end (grip-active ?other)))
    :effect (and (at start (grip-active ?r))
                 (at end (ready ?r)))))

(define (problem table-preparation-two-robots)
  (:domain table-preparation)
  (:objects r1 r2 - robot)
  (:init (partner r1 r2) (partner r2 r1))
  (:goal (and (ready r1) (ready r2))))
\end{verbatim}
}

\noindent M2 runs out of its four states, so a planner that took this as a proof would answer
\emph{unsolvable}. The validator accepts the plan \texttt{0.000: (prepare-lift r1 r2) [10]},
\texttt{1.000: (prepare-lift r2 r1) [10]} (makespan 11), and \sys\ returns
\texttt{0.000: (prepare-lift r1 r2) [10]}, \texttt{0.000: (prepare-lift r2 r1) [10]}
(makespan 10), found by M3 after M2 runs out of states. The task is in the test suite of
\sys.

\section{Full Tables}
\label{app:tables}
\begin{longtable}{@{}llrrr@{}}
\caption{Tasks solved within 100\,ms, 1\,s and the limit, by every planner, with uncorrected
times; in parentheses, the counts without the start-up of planutils's Apptainer images. ``of
$n$'': a run on a subset of the suite.}\label{tab:fullcoverage}\\
\toprule
Suite & Planner & 100\,ms & 1\,s & limit \\
\midrule
\endfirsthead
\toprule
Suite & Planner & 100\,ms & 1\,s & limit \\
\midrule
\endhead
Classical & \sys & \textbf{76} & \textbf{111} & \textbf{166} \\
 & LAMA & -- & -- & 160 \\
 & lama-first & 0 & 75 & 159 \\
 & Scorpion Maidu & -- & -- & 154 \\
 & Scorpion Maidu agile & 0 & 73 & 142 \\
 & Levitron & -- & -- & 156 \\
 & Levitron agile & 0 & 80 & 141 \\
 & FDSS 2023 & -- & -- & 157 \\
 & FDSS 2023 agile & 0 & 73 & 142 \\
 & DecStar 2023 & 0 & 72 & 146 \\
 & Dual-BFWS & 57 & 88 & 133 \\
\midrule
Numeric & \sys & \textbf{168} & \textbf{216} & \textbf{299} \\
 & ENHSP & 0 & 128 & 208 \\
 & ENHSP-2020 & 0 & 37 & 103 (0, 41) \\
 & LNM-Plan & 22 & 69 & 117 \\
 & Metric-FF & 46 & 51 & 57 \\
\midrule
Temporal & \sys & \textbf{102} & \textbf{147} & \textbf{180} \\
 & POPF & 0 & 82 & 104 (53, 84) \\
 & OPTIC & -- & -- & 107 \\
 & TFD & -- & -- & 110 \\
 & LPG-td & -- & -- & 125 \\
 & NextFLAP & -- & -- & 80 \\
\midrule
Concurrency & \sys & \textbf{25} & \textbf{32} & \textbf{33} \\
 & POPF & 0 & 18 & 24 (14, 19) \\
 & OPTIC & -- & -- & 17 \\
 & TFD & -- & -- & 24 \\
 & LPG-td & -- & -- & 0 \\
 & NextFLAP & -- & -- & 2 \\
\midrule
PDDL3 & \sys & \textbf{59} & \textbf{71} & \textbf{86} \\
 & POPF & 0 & 15 & 16 (13, 15) \\
 & OPTIC & -- & -- & 45 \\
 & NextFLAP & -- & -- & 10 \\
\midrule
PDDL+ & \sys & \textbf{55} & \textbf{94} & \textbf{148} \\
 & ENHSP & 0 & 18 & 23 \\
 & SMTPlan+ & 0 & 31 & 35 (25, 31) \\
\midrule
Durations & \sys & \textbf{24} & \textbf{24} & \textbf{25} \\
 & POPF & 0 & 17 & 17 (17, 17) \\
 & OPTIC & -- & -- & 5 \\
 & NextFLAP & -- & -- & 5 \\
\midrule
Hard to ground & \sys & \textbf{48} & \textbf{156} & \textbf{215} \\
 & Powerlifted & 22 & 106 & 190 \\
\midrule
PlanSys2 at scale & \sys & \textbf{60} & \textbf{77} & \textbf{92} \\
 & POPF & 0 & 45 & 57 (23, 45) \\
 & OPTIC & -- & -- & 57 \\
 & TFD & -- & -- & 58 \\
 & TAMER & 0 & 0 & 21 \\
 & Aries & -- & -- & 9 \\
\bottomrule
\end{longtable}

{\scriptsize
\begin{longtable}{@{}lp{8.2cm}rp{4.4cm}@{}}
\caption{Domains where another planner solves more tasks than \sys\ (at the limit).}
\label{tab:losses}\\
\toprule
Suite & Domain & \sys & Better \\
\midrule
\endfirsthead
\toprule
Suite & Domain & \sys & Better \\
\midrule
\endhead
Classical & \nolinkurl{ipc-2000/blocks-strips-typed} & 9/10 & Levitron 10, Scorpion Maidu 10 \\
Classical & \nolinkurl{ipc-2002/depots-strips-automatic} & 9/10 & Levitron 10, Scorpion Maidu 10, Dual-BFWS 10 \\
Classical & \nolinkurl{ipc-2004/pipesworld-no-tankage-nontemporal-strips} & 9/10 & Dual-BFWS 10 \\
Classical & \nolinkurl{ipc-2006/trucks-propositional-strips} & 5/10 & FDSS 2023 6, Levitron 6, Scorpion Maidu 6, Scorpion Maidu agile 6, Levitron agile 6 \\
Classical & \nolinkurl{ipc-2011/parking-sequential-satisficing} & 7/10 & LAMA 10, lama-first 10, FDSS 2023 8, Levitron 9, Scorpion Maidu 9, Dual-BFWS 10 \\
Classical & \nolinkurl{ipc2023/folding} & 2/10 & LAMA 3, lama-first 3, Levitron 3, Scorpion Maidu 3 \\
Classical & \nolinkurl{ipc2023/ricochet-robots} & 2/10 & DecStar 2023 3 \\
Classical & \nolinkurl{ipc2023/rubiks-cube} & 7/10 & LAMA 10, lama-first 10, DecStar 2023 10 \\
Numeric & \nolinkurl{ipc-2002/settlers-numeric-automatic} & 1/10 & Metric-FF 3 \\
Numeric & \nolinkurl{ipc2023-numeric/counters} & 6/10 & ENHSP 10 \\
Numeric & \nolinkurl{ipc2023-numeric/ext-plant-watering} & 5/10 & ENHSP 10 \\
Numeric & \nolinkurl{ipc2023-numeric/fo-sailing} & 0/10 & ENHSP 7, LNM-Plan 7 \\
Numeric & \nolinkurl{ipc2026-numeric/sailing-wind-opt} & 4/10 & ENHSP 10 \\
Numeric & \nolinkurl{ipc2026-numeric/sailing-wind-sat} & 0/10 & ENHSP 1 \\
Numeric & \nolinkurl{ipc2026-numeric/ztalloc-sum} & 1/10 & ENHSP 4 \\
Temporal & \nolinkurl{ipc-2002/rovers-time-automatic} & 4/5 & LPG-td 5, TFD 5 \\
Temporal & \nolinkurl{ipc-2004/satellite-time-time-windows-strips} & 2/5 & LPG-td 3 \\
Temporal & \nolinkurl{ipc-2006/pathways-metric-time} & 2/5 & LPG-td 3 \\
Temporal & \nolinkurl{ipc-2006/storage-time} & 3/5 & LPG-td 4 \\
Temporal & \nolinkurl{ipc-2008/parc-printer-temporal-satisficing-strips} & 3/5 & TFD 4 \\
Temporal & \nolinkurl{ipc-2011/floor-tile-temporal-satisficing} & 3/5 & LPG-td 4 \\
Temporal & \nolinkurl{ipc-2014/floor-tile-temporal-satisficing} & 2/5 & LPG-td 5 \\
Concurrency & \nolinkurl{ipc-2011/temporal-machine-shop-temporal-satisficing} & 0/5 & POPF 1 \\
Concurrency & \nolinkurl{ipc-2011/turn-and-open-temporal-satisficing} & 2/5 & TFD 5 \\
Concurrency & \nolinkurl{ipc-2014/turn-and-open-temporal-satisficing} & 2/5 & TFD 3 \\
PDDL3 & \nolinkurl{ipc-2006/pipesworld-metric-time-constraints} & 0/5 & OPTIC 1 \\
PDDL3 & \nolinkurl{ipc-2006/pipesworld-preferences-complex} & 2/5 & OPTIC 3, POPF 4, NextFLAP 3 \\
PDDL3 & \nolinkurl{ipc-2006/storage-preferences-complex} & 1/5 & OPTIC 4 \\
PDDL3 & \nolinkurl{ipc-2006/storage-preferences-qualitative} & 1/5 & OPTIC 4 \\
PDDL3 & \nolinkurl{ipc-2006/storage-preferences-simple-grounded-preferences} & 3/5 & OPTIC 4 \\
PDDL3 & \nolinkurl{ipc-2006/trucks-preferences-complex} & 4/5 & OPTIC 5 \\
PDDL+ & \nolinkurl{pddlplus/chasing_cars/CCy} & 1/6 & ENHSP 5 \\
Hard to ground & \nolinkurl{htg/labyrinth/SAT} & 0/5 & Powerlifted 3 \\
Hard to ground & \nolinkurl{htg/pipesworld-tankage-nosplit} & 4/5 & Powerlifted 5 \\
\bottomrule
\end{longtable}}

\section{Online Resources}
\label{app:resources}

The planner (\repo{https://github.com/PlanSys2/kairos}, Apache~2.0) and the benchmark
(\repo{https://github.com/PlanSys2/kairos-benchmarks}, Apache~2.0 for our problems and code).
The directory \texttt{papers/kairos2026} of the benchmark repository holds the scripts as they
were run, the raw results of every run of every planner, \sys's plans, the container recipes of
the other planners with their pinned commits, and the comparison with VAL of
Section~\ref{sec:val-agreement}, with the plan, domain and problem of each disagreement. Running
\texttt{bench/paper\_tables.py} there regenerates the tables of this article from those
results.

\end{document}